\documentclass{article}

\usepackage{microtype}
\usepackage{graphicx}
\usepackage{mathrsfs}
\usepackage{subfig}
\usepackage{booktabs} %

\usepackage{hyperref}

\usepackage[review]{icml2024}

\usepackage{amsmath}
\usepackage{amssymb}
\usepackage{mathtools}
\usepackage{amsthm}
\usepackage{enumitem}
\usepackage[capitalize,noabbrev]{cleveref}

\usepackage{wrapfig}
\usepackage{lipsum}  
\newcommand{\ind}{\perp\!\!\!\!\perp} 

\usepackage{mathtools}
\newcommand{\defeq}{\vcentcolon=}

\usepackage{algpseudocode}

\usepackage{xpatch}

\makeatletter
\xpatchcmd{\algorithmic}
  {\ALG@tlm\z@}{\leftmargin\z@\ALG@tlm\z@}
  {}{}
\makeatother

\usepackage{bbm}

\usepackage{tikz-cd}
\def\thetitle{Distributional Bellman Operators over Mean Embeddings}

\usepackage{thm-restate}
\newtheoremstyle{definition}
{3pt} %
{3pt} %
{} %
{} %
{\bfseries} %
{.} %
{.5em} %
{} %

\theoremstyle{plain}
\newtheorem{theorem}{Theorem}[section]

\theoremstyle{definition}
\newtheorem{definition}[theorem]{Definition}

\theoremstyle{remark}
\newtheorem{remark}[theorem]{Remark}

\usepackage[textsize=tiny]{todonotes}

\usepackage{amsmath,amsfonts,bm}

\def\eqref#1{equation~(\ref{#1})}
\def\Eqref#1{Equation~(\ref{#1})}

\def\1{\bm{1}}

\DeclareMathAlphabet{\mathsfit}{\encodingdefault}{\sfdefault}{m}{sl}
\SetMathAlphabet{\mathsfit}{bold}{\encodingdefault}{\sfdefault}{bx}{n}

\DeclareMathOperator*{\argmax}{arg\,max}
\DeclareMathOperator*{\argmin}{arg\,min}

\newcommand{\sketchfn}{U}
\newcommand{\bellmancoeff}{B}
\newcommand\numberthis{\addtocounter{equation}{1}\tag{\theequation}}

\icmltitlerunning{\thetitle}

\begin{document}

\twocolumn[
\icmltitle{\thetitle}

\begin{icmlauthorlist}
\icmlauthor{Li Kevin Wenliang}{dm}
\icmlauthor{Gr\'egoire Del\'etang}{dm}
\icmlauthor{Matthew Aitchison}{dm}
\icmlauthor{Marcus Hutter}{dm}
\icmlauthor{Anian Ruoss}{dm}
\icmlauthor{Arthur Gretton}{dm,ucl}
\icmlauthor{Mark Rowland}{dm}
\end{icmlauthorlist}

\icmlaffiliation{dm}{Google DeepMind}
\icmlaffiliation{ucl}{Gatsby Unit, University College London}

\icmlcorrespondingauthor{LKW}{\href{mailto:kevinliw@google.com}{kevinliw@google.com}}
\icmlcorrespondingauthor{MR}{\href{mailto:markrowland@google.com}{markrowland@google.com}}

\icmlkeywords{Machine Learning, ICML}

\vskip 0.3in
]

\printAffiliationsAndNotice{}  %

\begin{abstract}
    We propose a novel algorithmic framework for distributional reinforcement learning, based on learning finite-dimensional mean embeddings of return distributions. 
    The framework reveals a wide variety of new algorithms for dynamic programming and temporal-difference algorithms that rely on the \emph{sketch Bellman operator}, which updates mean embeddings with simple linear-algebraic computations.
    We provide asymptotic convergence theory, and examine the empirical performance of the algorithms on a suite of tabular tasks.
    Further, we show that this approach can be straightforwardly combined with deep reinforcement learning.
\end{abstract}

\section{Introduction}\label{sec:intro}

In distributional approaches to reinforcement learning (RL), the aim is to learn the full probability distribution of future returns \citep{morimura2010nonparametric,bellemare2017distributional,bdr2023}, rather than just their expected value, as is typically the case in value-based reinforcement learning \citep{sutton2018reinforcement}.
Distributional RL was proposed in the setting of deep reinforcement learning by \citet{bellemare2017distributional}, with a variety of precursor work stretching back almost as far as Markov decision processes themselves \citep{jaquette1973markov,sobel1982variance,chung87discounted,morimura2010nonparametric,morimura2010parameteric}. Beginning with the work in \citet{bellemare2017distributional}, the distributional approach to reinforcement learning has been central across a variety of applications of deep RL in simulation and in the real world \citep{bodnar2020quantile,bellemare2020autonomous,wurman2022outracing,fawzi2022discovering}. 

Typically, predictions of return distributions are represented directly as approximate probability distributions, such as categorical distributions
\citep{bellemare2017distributional}.
\citet{rowland2019statistics}
proposed an alternative framework 
where return distributions are represented via 
the values of \emph{statistical functionals}, called a \emph{sketch} by \cite{bdr2023}.
This provided a new space of distributional reinforcement learning algorithms, leading to improvements in deep RL agents, and hypotheses regarding distributional RL in the brain \citep{dabney2020distributional,lowet2020distributional}.
On the other hand, a potential drawback of this approach is that each 
distributional Bellman update to the representation, these values must be ``decoded'' back into an approximate distribution via an \emph{imputation strategy}. In practice, this can introduce significant computational overhead to Bellman updates, 
and is unlikely to be biologically plausible for distributional learning in the brain \citep{tano2020local}.

Here, we focus on a notable instance of the sketch called the \emph{mean embedding sketch}.
In short, the mean embedding is the expectation of nonlinear functions under the distribution represented \citep{SmoGreSonSch07,SriGreFukLanetal10,berlinet2011reproducing},
and is related to \emph{frames} in signal processing \citep{mallat1999wavelet} and 
\emph{distributed distributional code} 
in theoretical neuroscience \citep{sahani2003doubly,vertes2018flexible}.
The core contributions of this paper are to revisit the approach to distributional reinforcement learning based on sketches \citep{rowland2019statistics}, and to propose the \emph{sketch Bellman operator} that updates the implicit
distributional representation as a simple linear operation, obviating the need for the expensive imputation strategies converting between sketches and distributions.
This leads to a rich new space of distributional RL algorithms that operate entirely in the space of sketches.
We provide theoretical convergence analysis
to accompany the framework, 
investigate the practical behaviour of various instantiations of the proposed algorithms in tabular domains, and demonstrate the effectiveness of the sketch framework in deep reinforcement learning.
\section{Background}\label{sec:background}
We consider a Markov decision process (MDP) with state space $\mathcal{X}$, action space $\mathcal{A}$, transition probabilities $P : \mathcal{X} \times \mathcal{A} \rightarrow \mathscr{P}(\mathcal{X})$, reward distribution function $P_R : \mathcal{X} \times \mathcal{A} \rightarrow \mathscr{P}(\mathbb{R})$, and discount factor $\gamma \in [0, 1)$. Given a policy $\pi : \mathcal{X} \rightarrow \mathscr{P}(\mathcal{A})$ and initial state $x \in \mathcal{X}$, a random trajectory $(X_t, A_t, R_t)_{t \geq 0}$ is the sequence of random states, actions, and rewards encountered when using the policy $\pi$ to select actions in this MDP. More precisely, we have $X_0 = x$, $A_t \sim \pi(\cdot|X_t)$, $R_t \sim P_R(X_t, A_t)$, $X_{t+1} \sim P(\cdot|X_t, A_t)$ for all $t \geq 0$. We write $\mathbb{P}^\pi_x$ and $\mathbb{E}^\pi_x$ for probabilities and expectations with respect to this distribution, respectively (conditioned on $X_0=x$).
The performance along the trajectory is measured by the discounted return, defined by
\begin{align}\label{eq:random-return}
    \sum_{t=0}^\infty \gamma^t R_t \, .
\end{align}
In typical value-based RL, during policy evaluation, the agent learns the expectation of the return for each possible initial state $x \in \mathcal{X}$, which is encoded by the value function $V^\pi : \mathcal{X} \rightarrow \mathbb{R}$, given by $V^\pi(x) = \mathbb{E}^\pi_x[\sum_{t=0}^\infty \gamma^t R_t ]$.

\subsection{Distributional RL and Bellman equation}

In distributional RL,
the problem of policy evaluation is to learn the probability distribution of return in \Eqref{eq:random-return}
for each possible initial state $x \in \mathcal{X}$. This is encoded by the return-distribution function $\eta^\pi : \mathcal{X} \rightarrow \mathscr{P}(\mathbb{R})$, which maps each initial state $x \in \mathcal{X}$ to the corresponding distribution of the random return.
A central result in distributional reinforcement learning is the distributional Bellman equation, which relates the distribution of the random return under different combinations of initial states and actions.

To build the random variable formulation of the returns, we let $(G^\pi(x) : x \in \mathcal{X})$ be a collection of random variables with the property that $G^\pi(x)$ 
is equal to \Eqref{eq:random-return} in distribution, conditioned on the initial state $X_0=x$. 
This formulation implies that the random variable $G^\pi(x)$ is distributed as $\eta^\pi(x)$
for all $x\in\mathcal{X}$.
Consider a random transition $(x, R, X')$ generated by $\pi$,
independent of the $G^\pi$ random variables.
Then, the (random variable) distributional Bellman equation states that for each state $x$, 
\begin{align*}
    G^\pi(x) \overset{\mathcal{D}}{=} R + \gamma {G}^\pi(X') \quad  | \, X = x\, .
\end{align*}
Here, we use the slight abuse of the conditioning bar to set the distribution of $X$ in the random transition. It is also useful to introduce the distributional Bellman operator $\mathcal{T}^\pi : \mathscr{P}(\mathbb{R})^\mathcal{X} \rightarrow \mathscr{P}(\mathbb{R})^\mathcal{X}$ to describe the transformation that occurs on the right-hand side \citep{morimura2010nonparametric,bellemare2017distributional}. If $\eta \in \mathscr{P}(\mathbb{R})^\mathcal{X}$ is a collection of probability distributions, and $(G(x) : x \in \mathcal{X})$ is a collection of random variables such that $G(x) \sim \eta(x)$ for all $x$, and $(X, R, X')$ is random transition generated by $\pi$, independent of $(G(x) : x \in \mathcal{X})$, then $(\mathcal{T}^\pi \eta)(x) = \text{Dist}(R + \gamma G(X') | X = x)$.

To implement algorithms of distributional RL, one needs to approximate the infinite-dimensional return-distribution function $\eta^\pi$ with finite-dimensional representations. This is typically done via direct approximations in the space of distributions; see e.g. \citet[Chapter~5;][]{bdr2023}.

\subsection{Statistical functionals and sketches}

Rather than using approximations in the space of distributions, \citet{rowland2019statistics} proposed to represent return distributions indirectly via \emph{functionals} of the return distribution, called \emph{sketches} by \citet{bdr2023}. In this work we consider a specific class of sketches, defined below.

\begin{definition}[Mean embedding sketches]
    A \emph{mean embedding sketch} $\psi$ is specified by a function $\phi : \mathbb{R} \rightarrow \mathbb{R}^m$, and defined by
    \begin{align}\label{eq:mean embedding}
        \psi(\nu) ~\defeq~ \mathbb{E}_{Z \sim \nu} [\phi(Z)] \, .
    \end{align}
\end{definition}
For a given distribution $\nu$, the embedding $\psi(\nu)$ can therefore be thought of as providing a \emph{lossy} summary of the distribution. The name is motivated by the kernel literature, in which \Eqref{eq:mean embedding} can be viewed as embedding the distribution $\nu$ into $\mathbb{R}^m$ based on the mean of $\phi$ under $\nu$ 
\citep{SmoGreSonSch07,SriGreFukLanetal10,berlinet2011reproducing}. As we will show, the mean embedding sketch
enables elegant distributional RL algorithms.

\begin{figure}[t]
    \centering
    \includegraphics[keepaspectratio,width=.45\textwidth]{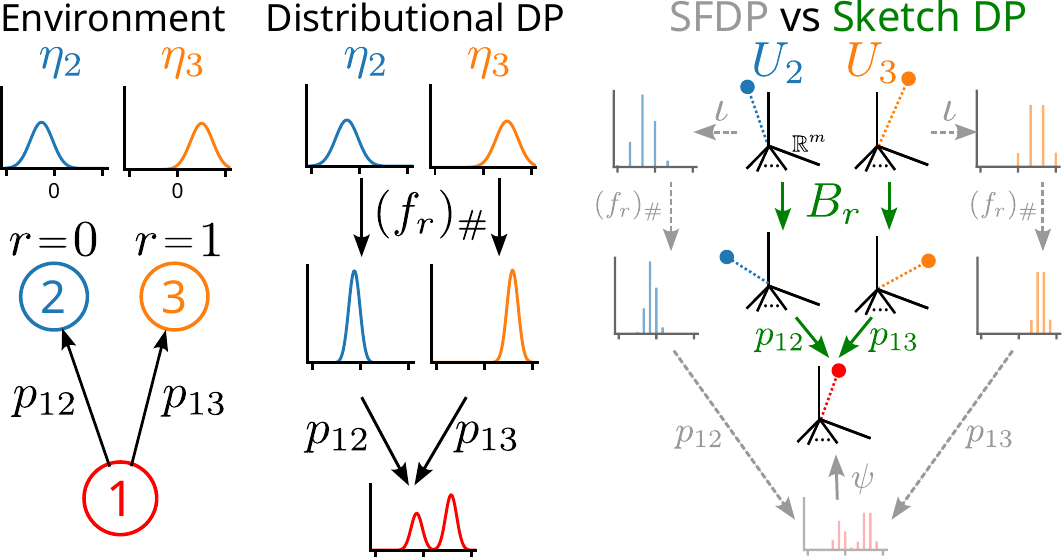}
    \caption{Example of DP update 
    for state 1 with child states 2 and 3
    and return distributions $\eta_2$ and $\eta_3$.
    In the exact distributional DP, $\eta$'s are 
    scaled and shifted by $f_{r}(g)=r+\gamma g$,
    and then weighted by the transition probabilities. In SFDP
    \citep{rowland2019statistics}, 
    the map $\iota$ imputes, from initial sketch 
    values $U$, approximate (e.g. categorical) distributions  
    on which the distribution DP is applied, followed by evaluating
    the sketch $\psi$. 
    In our approach Sketch-DP, 
    the updates are computed in the mean embedding space,
    facilitated by the Bellman coefficients $B_r$,
    avoiding the imputation step.
    }
    \label{fig:background-diagram}
\end{figure}

Statistical functional dynamic programming (SFDP; \citet{rowland2019statistics}, see also \citet{bdr2023}) is an approach to distributional RL in which sketch values, rather than approximate distributions, are the primary object learned. 
Given a sketch $\psi$ and estimated sketch values $\sketchfn : \mathcal{X} \rightarrow \mathbb{R}^m$, 
SFDP proceed by first defining an \emph{imputation strategy} $\iota : \mathbb{R}^m \rightarrow \mathscr{P}(\mathbb{R})$ mapping sketch values back to distributions, with the aim that $\psi (\iota(\sketchfn)) \approx \sketchfn$,
so that $\iota$ acts as an approximate pseudo-inverse of $\psi$. 
The usual Bellman backup is then applied to this \emph{imputed} distribution, and the sketch value extracted from this updated distribution. Thus,
a typical update in SFDP takes the form $\sketchfn \leftarrow \psi((\mathcal{T}^\pi \iota (\sketchfn))(x))$; see Figure~\ref{fig:background-diagram}.

This approach led to expectile-regression DQN, a deep RL agent that aims to learn the sketch values associated with certain expectiles \citep{newey1987asymmetric} of the return, and influenced a distributional model of dopamine signalling in the brain \citep{dabney2020distributional}. An important consideration is that computation of the imputation strategy is often costly in applications and considered biologically implausible in neuroscience models \citep{tano2020local}.

\section{The Bellman sketch framework}\label{sec:framework}

Our goal is to derive a framework for approximate computation of the sketch $\psi$ (with corresponding feature function $\phi$) of the return distributions corresponding to a policy $\pi$, without needing to design, implement, or compute an imputation strategy as in the case of SFDP/TD. 
That is, we aim to compute the object $\sketchfn^\pi : \mathcal{X} \rightarrow \mathbb{R}^m$, given by
\begin{align*}
    \textstyle
    \sketchfn^\pi(x) ~\defeq~ \psi(\eta^\pi(x)) ~=~  \mathbb{E}^\pi_x[\phi(\sum_{t=0}^\infty \gamma^t R_t)] \, .
\end{align*}
We begin by considering environments with
a finite set of possible rewards
$\mathcal{R} \subseteq \mathbb{R}$;
we discuss 
generalisations later.
To motivate our method, we consider a special case;
suppose that for each possible return $g \in \mathbb{R}$, and each possible immediate reward $r \!\in \!\mathcal{R}$, there exists a matrix $\bellmancoeff_r$ such that
\begin{align}\label{eq:linear-bellman-closed}
    \phi(r + \gamma g) ~=~ \bellmancoeff_r \phi(g) \, ;
\end{align}
note that $\bellmancoeff_r$ does not depend on $g$, and $\gamma$ is a constant. In words, this says that the feature function $\phi$ evaluated at the bootstrap return $r + \gamma g$ is expressible as a linear transformation of the feature function evaluated at $g$ itself.
If such a relationship holds, then we have
\begin{align*}
    \sketchfn^\pi(x)
     &\overset{(a)}{=} \mathbb{E}^\pi_x[ \phi(R + \gamma G^\pi(X')) ] 
     \overset{(b)}{=} \mathbb{E}^\pi_x[ \bellmancoeff_R \phi(G^\pi(X')) ]  \\
     &\overset{(c)}{=} \mathbb{E}^\pi_x[ \bellmancoeff_R \sketchfn^\pi(X') ] \, ,\numberthis\label{eq:lbec}
\end{align*}
where (a) follows from the distributional Bellman equation, (b) follows from \Eqref{eq:linear-bellman-closed}, and (c) from exchanging the linear map $\bellmancoeff_r$ and the conditional expectation given $(R, X')$,
crucially relying on the linearity of the approximation in \Eqref{eq:linear-bellman-closed}.
Note that for example with $\phi(g)=(1, g)^\top$ we have $\bellmancoeff_r=({1~0\atop r~\gamma})$, and \Eqref{eq:lbec} reduces to the classical Bellman equation for $V^\pi$, with $U^\pi(x)=(1, V^\pi(x))^\top$.

Thus, $\sketchfn^\pi(x)$ satisfies its own linear Bellman equation, which motivates algorithms that work directly in the space of sketches, without recourse to imputation strategies. In particular, a natural dynamic programming algorithm to consider is based on the recursion
\begin{align}\label{eq:dp}\tag{Sketch-DP}
    \sketchfn(x) ~\leftarrow~ \mathbb{E}^\pi_{x}[ \bellmancoeff_R \sketchfn(X') ] \, .
\end{align}

See \cref{fig:background-diagram} for an example 
and comparison with SFDP. As this is an update applied directly to sketch values themselves, we introduce the \emph{sketch Bellman operator} $\mathcal{T}^\pi_\phi : (\mathbb{R}^{m})^\mathcal{X} \rightarrow (\mathbb{R}^{m})^\mathcal{X}$, with $(\mathcal{T}^\pi_\phi \sketchfn)(x)$ defined according to the right-hand side of \Eqref{eq:dp}. Note that $\mathcal{T}^\pi_\phi$ is a \emph{linear} operator, in contrast to the standard expected-value Bellman operator, which is affine. We recover the affine case by taking one component of $\phi$ to be constant, e.g. $\phi_1(g) \equiv 1$, and enforcing $U_1(x) \equiv 1$.

The right-hand side of \Eqref{eq:dp} can be unbiasedly approximated with a sample transition $(x, r, x')$. Stochastic approximation theory \citep{kushner1997stochastic,bertsekas1996neuro} then naturally suggests the following temporal-difference learning update:
\begin{align}\label{eq:td}\tag{Sketch-TD}
    \sketchfn(x) ~\leftarrow~ (1-\alpha) \sketchfn(x) + \alpha \bellmancoeff_r U(x') \, 
\end{align}

given a learning rate $\alpha$.
\citet{rowland2019statistics} introduced the term \emph{Bellman closed} for sketches for which an \emph{exact} dynamic programming algorithm is available, and provided a characterisation of Bellman closed mean embedding sketches. The notion of Bellman closedness is closely related to the relationship in \Eqref{eq:linear-bellman-closed}, and from \citet[Theorem~4.3; ][]{rowland2019statistics}, we can deduce that
the only mean embedding sketches that satisfy \Eqref{eq:linear-bellman-closed} are invertible linear combinations of the first-$m$ moments.

Thus, our discussion above serves as a way of re-deriving known algorithms for computing moments of the return \citep{sobel1982variance,lattimore2014near}, but is insufficient to yield algorithms for computing other sketches. Additionally, since moments of the return distribution are naturally of widely differing magnitudes, it is difficult to learn a high-dimensional mean embedding based on moments; see Appendix~\ref{sec:polynomial_tabular} for further details. To go further, we must weaken the assumption made in \Eqref{eq:linear-bellman-closed}.

\subsection{General sketches}\label{sec:genreal_sketch_dp_td}

To extend our framework to a much more general family of sketches, we relax our assumption of the exact predictability of $\phi(r + \gamma g)$ from $\phi(g)$ in \Eqref{eq:linear-bellman-closed}, by defining a matrix of \emph{Bellman coefficients} $\bellmancoeff_r$ for each possible reward $r \in \mathcal{R}$ as the solution of the linear regression problem:
\begin{align}\label{eq:regression}
    \bellmancoeff_r ~\defeq~ \argmin_{\bellmancoeff}\mathbb{E}_{G \sim \mu}\Big[ \|\phi(r + \gamma G) - \bellmancoeff \phi(G) \|_2^2 \Big] \, ,
\end{align}
so that, informally, we have $\phi(r + \gamma g) \approx \bellmancoeff_r \phi(g)$ for each $g$. Here, $\mu$ is a distribution to be specified that weights the returns $G$. Using the same motivation as in the previous section, we therefore obtain
\begin{align*}
    \sketchfn^\pi(x)
     &\overset{(a)}{=} \mathbb{E}^\pi_x[ \phi(R + \gamma G^\pi(X')) ] 
     \approx \mathbb{E}^\pi_x[ \bellmancoeff_R \phi(G^\pi(X')) ]  \\
     &\overset{(c)}{=} \mathbb{E}^\pi_x[ \bellmancoeff_R \sketchfn^\pi(X') ] \, ,\numberthis\label{eq:sketch-derivation}
\end{align*}
noting that informally we have \emph{approximate} equality in the middle of this line. This still motivates the approaches expressed in Equations~(\ref{eq:dp}) and~(\ref{eq:td}), though we have lost the property that the exact sketch values $\sketchfn^\pi$ are a fixed point of the dynamic programming procedure.

\algrenewcommand\algorithmicindent{0.5em}%

    \begin{minipage}{1\columnwidth}
        \vspace{-0.5cm}
        \begin{algorithm}[H]
            \caption{Sketch-DP/Sketch-TD}
            \label{alg:dp}
            \begin{algorithmic}
                \State \textcolor{gray}{\# Precompute Bellman coefficients}
                \State Compute $C$ as in \Eqref{eq:AB}
                \For{$r \in \mathcal{R}$}
                    \State Compute $C_r$ as in \Eqref{eq:AB}
                    \State Set $\bellmancoeff_r =  C_r C^{-1}$
                \EndFor
                \State Initialise $\sketchfn : \mathcal{X} \rightarrow \mathbb{R}^{m}$
                \State \textcolor{gray}{\# Main loop}
                \For{$k=1,2,\ldots$}
                    \If{DP}
                        \State $\sketchfn(x) \leftarrow\! \sum\limits_{r, x', a} P(r, x'|x, a) \pi(a|x) \bellmancoeff_r \sketchfn(x')\ \ \forall x$
                    \ElsIf{TD}
                        \State Observe transition $(x_k, a_k, r_k, x'_k)$.
                        \State $\sketchfn(x_k) \leftarrow (1-\alpha_k) \sketchfn(x_k) + \alpha_k  \bellmancoeff_{r_k} \sketchfn(x'_k)$
                    \EndIf
                \EndFor
            \end{algorithmic}
        \end{algorithm}
    \end{minipage}
    
\textbf{Computing Bellman coefficients.} 
Under mild conditions (invertibility of $C$ as follows) the matrix of Bellman coefficients $\bellmancoeff_r$ defined in \Eqref{eq:regression} can be solved as $\bellmancoeff_r = C_r C^{-1}$, where $C, C_r \in \mathbb{R}^{m \times m}$ are defined by
\begin{equation}\label{eq:AB}
\begin{aligned}
  & C \defeq \mathbb{E}_{G \sim \mu}[\phi(G)\phi(G)^\top] \, ,\\
  & C_r \defeq  \mathbb{E}_{G \sim \mu}[\phi(r+\gamma G) \phi(G)^\top] \, .
\end{aligned}
\end{equation}
A derivation is in \Cref{sec:computational-properties} where we also describe the choice of $\mu$. The elements of these matrices are expressible as integrals over the real line, and hence several possibilities are available for (approximate) computation: if $\mu$ is finitely-supported, direct summation is possible; in certain cases the integrals may be analytically available, and otherwise numerical integration can be performed. Additionally, for certain feature maps $\phi$, the Bellman coefficients $\bellmancoeff_r$ have particular structure that can be exploited computationally; see Appendix~\ref{sec:regression-details} for further discussion. 
The generalisation to handle an infinite $\mathcal{R}$ is presented in \Cref{sec:knowledge_of_rewards}; 
and detailed properties of $B_r$ are 
studied in \cref{sec:bellman-coeffs-details}.

\textbf{Algorithms.}
We summarise the two core algorithmic contributions, \textbf{sketch dynamic programming} (Sketch-DP) and \textbf{sketch temporal-difference learning} (Sketch-TD), that arise from our proposed framework in Algorithm~\ref{alg:dp}. 
Pausing to take stock, we have proposed an algorithm framework for computing approximations of \emph{lossy} mean embeddings for a wide variety of feature functions $\phi$. Further, these algorithms operate directly within the space of sketch values.

\textbf{Selecting feature maps.}
A natural question is what effect the choice of feature map $\phi$ has on the performance of the algorithm. There are several competing concerns. First, the richer the map $\phi$, the more information about the return distribution can be captured by the corresponding mean embedding. However, in the worst case, the computational costs of our proposed \cref{alg:dp} scale cubically with $m$ (the dimensionality of $\phi$) prior to the iterative updates which then scale quadratically with $m$. In addition, the accuracy of the algorithm in approximating the mean embeddings of the true return distributions relies on having a low approximation error in \Eqref{eq:sketch-derivation}, which in turn relies on a low regression error in \Eqref{eq:regression} (see Proposition~\ref{prop:one-step} below). Selecting an appropriate feature map is therefore somewhat nuanced, and involves trading off a variety of computational and approximation concerns.

A collection of feature maps that offer the potential for trade-offs along the dimensions above is the
translation family
\begin{align}\label{eq:translation-family}
    \phi_i(z) \defeq \kappa( s(z - z_i) ), ~~\forall ~i\in \{1, \cdots, m\} \, ,
\end{align}
where $\kappa: \mathbb{R}\to\mathbb{R}$ is a \emph{base feature function}, $s \in \mathbb{R}^+$ is the \emph{slope}, and 
the set $\{z_1,\ldots,z_m\} \subseteq \mathbb{R}$
is the \emph{anchors} of the feature map. 
We will often take $\kappa$ to be commonly used bounded and smooth nonlinear functions, such as the Gaussian or the sigmoid functions, and spread the anchor points over the return range. We emphasise that the choice of feature maps for the Bellman sketch framework is flexible; see \cref{sec:feature-type-details} for other possible choices. 

\begin{remark}[Invariance]%
    Given the $m$-dimensional function space obtained from the span of the coordinate functions $\phi_1,\ldots,\phi_m$, the algorithms above are essentially independent of the choice of basis for this space.
    For any invertible matrix $M \in \mathbb{R}^{m \times m}$, replacing $\phi$ by $M^{-1}\phi$, and \emph{also} $||\cdot||_2$ by $||\cdot||_{M^\top M}$ in \Eqref{eq:regression} gives an equivalent algorithm.    
\end{remark}

\begin{remark}[The need for \emph{linear} regression]\label{remark:linear-regression}
    It is tempting to try and obtain a more general framework by allowing \emph{non-linear} regression of $\phi(r + \gamma g)$ on $\phi(g)$ in \Eqref{eq:regression}, to obtain a more accurate fit, for example fitting a function $H : \mathbb{R} \times \mathbb{R}^m \rightarrow \mathbb{R}^m$ so that $\phi(r +\gamma g) \approx H(r, \phi(g))$. The issue is that if $H$ is not linear in the second argument, then generally $\mathbb{E}[H(r, \phi(G(X')))] \not= H(r, \mathbb{E}[\phi(G(X'))])$, and so step (c) in \Eqref{eq:sketch-derivation} is not valid.
    However, there may be settings where it is desirable to \emph{learn} such a function $H$, to avoid online computation of Bellman coefficients every time a new reward value is encountered in TD learning.
\end{remark}
\begin{remark}[Linear update] 
The sketch updates in \cref{eq:dp,eq:td} are linearly, which is distinct from typical particle-based distributional RL algorithms (e.g. \citep{dabney2018distributional, dabney2018implicit, nguyen2020distributional}, 
where the updates involve non-linear operations. In particular, \citet{nguyen2020distributional} proposed a TD algorithm
for updating particle locations by decreasing a sample MMD objective \citep{gretton2012kernel}. However, this does not yield a dynamic programming algorithm, and \citet{nguyen2020distributional} do not analyse the TD algorithm; further comparisons are more clearly described in Appendix~\ref{sec:comparison_details}.
\end{remark}

\begin{figure*}[h]
    \centering
    \begin{minipage}{0.18\textwidth}
    \footnotesize
    \textbf{A}\\
    \includegraphics[width=0.9\textwidth, trim=-2cm -1.5cm 0 0, clip=true]{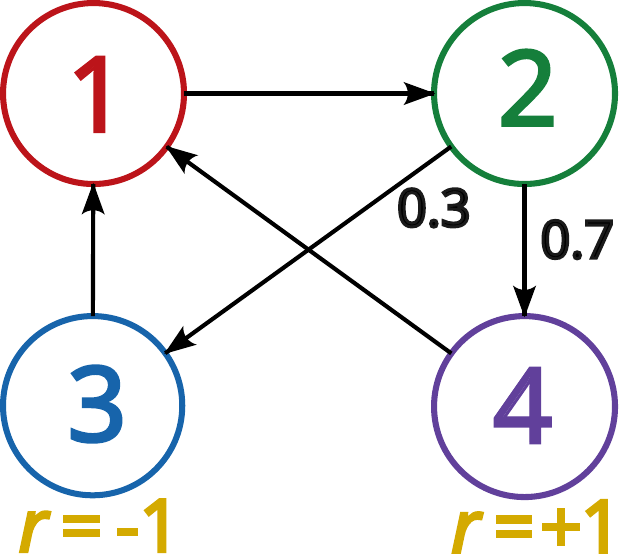}
    \vspace{-0.3cm}\\
    \textbf{B}\\
    \includegraphics[width=0.9\textwidth]{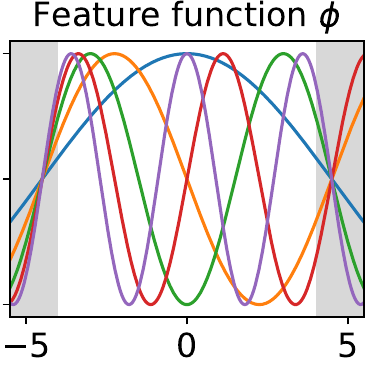}
    \end{minipage}
    \begin{minipage}{0.31\textwidth}
        \textbf{C}\\
        \includegraphics[width=0.9\columnwidth]{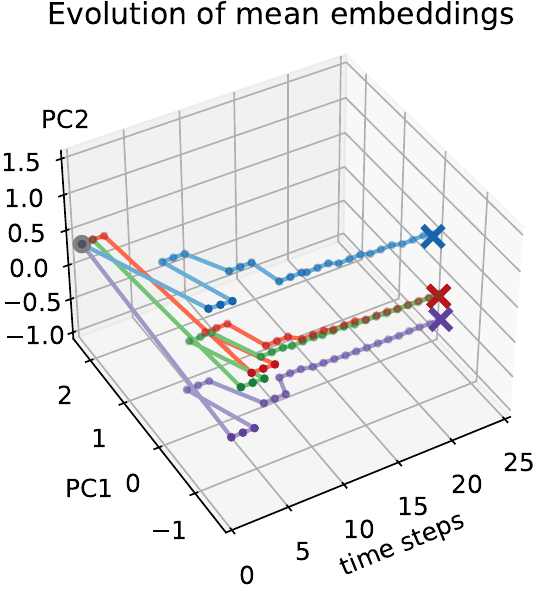}
    \end{minipage}
    \hspace{0.3cm}
    \begin{minipage}{0.47\textwidth} 
        \footnotesize
        \hspace{-0.4cm}\textbf{D} \vspace{-1.2mm}\hspace{2.7cm}\textbf{E}\\
        \includegraphics[width=0.3\textwidth]{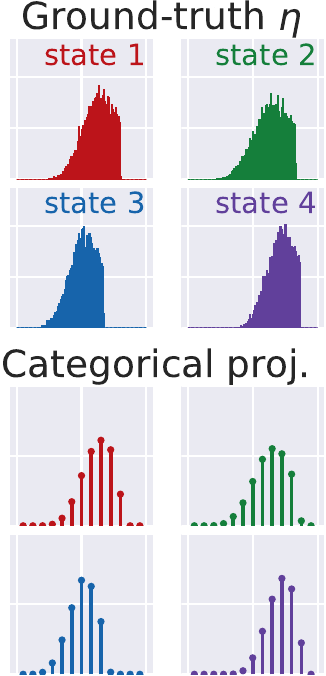}~~~~~~
        \includegraphics[width=0.3\textwidth]{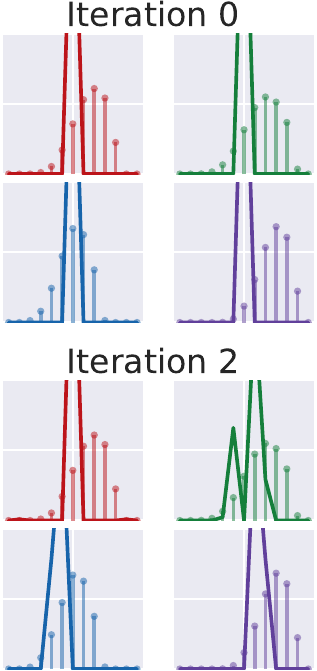}~~~
        \includegraphics[width=0.3\textwidth]{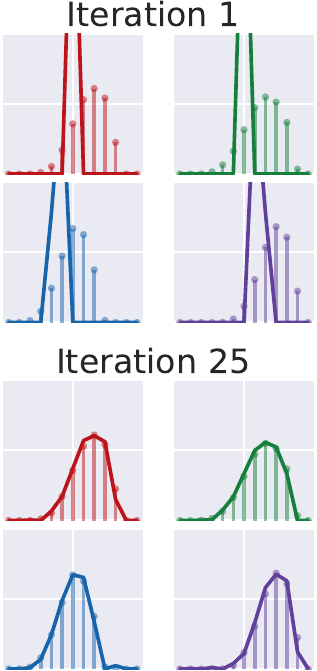}
        \\
    \end{minipage}
    \caption{
        An example run of Sketch-DP. 
        \textbf{A}, The MRP considered here.
        \textbf{B}, The first 5 of $m=13$ sinusoidal feature functions $\phi$.
        The regression \Eqref{eq:regression} is performed under a densely spaced grid over the white region $[-4, 4]$.
        \textbf{C}, Evolution of the estimated mean embeddings from initialisation (grey dot) onto the first two principal components. 
        Crosses represent the ground-truth mean embeddings. 
        \textbf{D}, Ground-truth return distributions (estimated by Monte-Carlo) and 
        their categorical projections onto a regular grid.
        \textbf{E}, Imputed distributions from the mean embeddings onto the same grid for selected 
        iterations (curves), compared against the categorical projections (stems). 
        }
        \vspace{-0.5cm}
    \label{fig:walk_through}
\end{figure*}
\subsection{Sketch-DP at work}\label{sec:fork-example}

To give more intuition for the Bellman sketch framework, we provide a walk-through of 
using \cref{alg:dp} to estimate the return distributions for the environment in \cref{fig:walk_through}A.
We take a sinusoidal feature map $\phi$ that consists of $m=13$ harmonics over the range $[-4.5, 4.5]$ (see \cref{sec:walk_through_details} for 
an example using feature map of the form \Eqref{eq:translation-family}). The Bellman regression problem in \Eqref{eq:regression} is set with $\mu = \text{Uniform}([-4, 4])$, 
based on the typical returns observed in the environment.
We then run the Sketch-DP algorithm with the initial estimates $U(x)$ set to $\phi(0)$ for all $x\in\mathcal{X}$.

To visualise how the estimated mean
embeddings evolve over iterations, we project them onto their first two principal components in \cref{fig:walk_through}C,
which explain $>70\%$ variance. To approximate the ground truth return distributions, we collected 
a large number of Monte Carlo samples from the MRP \cref{fig:walk_through}D; see \cref{sec:tabular-details} for details. 
We then estimate the ground-truth mean
embeddings and project them on to the principal subspace in \cref{fig:walk_through}C as crosses. 
The Sketch-DP estimates converge to close proximity of the ground-truth. 
The distinctive update pattern stems from the fact that all paths between rewarding states have length 3. 
The mean embeddings of states 1 and 2 are closer to state 4  due to more frequent transitions from state 2 to 4.

To aid interpretation of these results, we also include a comparison in which we ``decode'' the mean embeddings at 
selected iterations back into probability distributions (via an imputation strategy \citep{rowland2019statistics}), and compare with the ground-truth return distributions projected onto the anchor locations of the features \citep{rowland2018analysis}, as shown in \cref{fig:walk_through}D. Full details of the imputation strategy are in \cref{sec:cat-imputation}.
These decoded distributions are shown in Figure~\ref{fig:walk_through}E. Initially, the imputed distributions of the Sketch-DP mean embedding estimates reflect the initialisation to the mean embedding of $\delta_0$. As more iterations of Sketch-DP are applied, the imputed distributions of the evolving mean embedding estimates become close to the projected ground-truth. This indicates that, in this example, not only does Sketch-DP compute accurate mean embeddings of the return, but that this embedding is rich enough to recover a lot of information regarding the return distributions themselves.

Concluding the introduction of the Sketch-DP algorithmic framework, there are several natural questions that arise.
Can we quantify how accurately Sketch-DP algorithms can approximate mean embeddings of return distributions?
What effects do choices such as the feature map $\phi$ have on the algorithms in practice?
The next sections are devoted to answering these questions in turn.

\section{Convergence analysis}\label{sec:convergence}
We analyse the Sketch-DP procedure described in Algorithm~\ref{alg:dp}, 
with a novel error analysis approach that can be mathematically described in the following succinct manner. We let $\sketchfn_0 : \mathcal{X} \rightarrow \mathbb{R}^{m}$ denote the initial sketch value estimates, and then note from Algorithm~\ref{alg:dp} that the collection of estimates after each DP update form a sequence $(\sketchfn_k)_{k=0}^\infty$, with $\sketchfn_{k+1} = \mathcal{T}^\pi_\phi \sketchfn_k$. Our convergence analysis therefore focuses on the asymptotic behaviour of this sequence.
We introduce the notation $\Phi : \mathscr{P}(\mathbb{R}) \rightarrow \mathbb{R}^m$ for the sketch associated with the feature function $\phi$, so that $\Phi \mu = \mathbb{E}_{Z \sim \mu}[\phi(Z)]$, and define $\Phi$ for return-distribution functions (RDFs) by specifying for $\eta \in \mathscr{P}(\mathbb{R})^\mathcal{X}$ that $(\Phi \eta)(x) = \Phi (\eta(x))$.
Ideally, we would like these iterates to approach $\sketchfn^\pi : \mathcal{X} \rightarrow \mathbb{R}^{m}$, the sketch values of the true return distributions, given by $\sketchfn^\pi(x) = \mathbb{E}^\pi_{x}[\phi(\sum_{t=0}^\infty \gamma^t R_t)]$. As already described, typically this is not possible when the sketch $\Phi$ is not Bellman closed, and so we can only expect to approximate $\sketchfn^\pi$. Mathematically, this is because in general $\Phi \mathcal{T}^\pi \not= \mathcal{T}^\pi_\phi \Phi$ when $\phi$ is not Bellman closed.

The first step is to bound the error incurred in a single step of dynamic programming due to using $\mathcal{T}^\pi_\phi$ directly on the sketch values, rather taking sketch values after applying the true distributional Bellman operator to the underlying distributions; this corresponds to the foreground of Figure~\ref{fig:convergence-objects}.

\begin{restatable}{proposition}{propRegressionBound}(\textbf{Regression error to Bellman approximation.})\label{prop:one-step}
    Let $\|\cdot\|$ be a norm on $\mathbb{R}^m$. Then for any RDF $\eta \in \mathscr{P}([G_{\text{min}}, G_{\text{max}}])^{\mathcal{X}}$, we have
    \begin{align}\label{eq:sup_regression_loss}
        \max_{x \in \mathcal{X}} \|& \Phi (\mathcal{T}^\pi \eta)(x) - (\mathcal{T}^\pi_\phi \Phi \eta)(x) \|\\
        & \leq \sup_{g \in [G_\text{min}, G_{\text{max}}]} \max_{r \in \mathcal{R}} \| \phi(r + \gamma g) - \bellmancoeff_r \phi(g) \| \, . \nonumber
    \end{align}
\end{restatable}

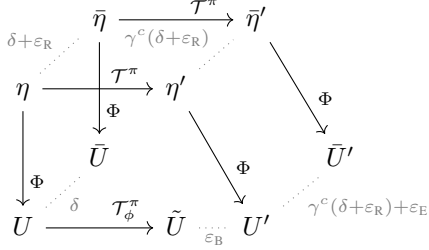
\begin{figure}[t]
    \centering
    \vspace{-0.72cm}
    \begin{equation*}
        \begin{tikzcd}[sep=small]%
                                                              & \bar{\eta} \arrow[rr,near end,"\mathcal{T}^\pi"] \arrow[dd,near end,"\Phi"] &                  & \bar{\eta}' \arrow[ddr,near end,"\Phi"]  &            \\
        \eta \arrow[rr,near end,"\mathcal{T}^\pi"] \arrow[dd,near end,"\Phi"] \arrow[ur,dotted,dash,gray,"\delta+\varepsilon_\text{R}"]    &                  & \eta' \arrow[ur,dotted,dash,gray,"\gamma^c(\delta+\varepsilon_\text{R})"] \arrow[ddr,near end,"\Phi"]    &              &            \\
                                                            & \bar{\sketchfn}    &              &              & \bar{\sketchfn}' \\
        \sketchfn \arrow[rr,near end,"\mathcal{T}^\pi_\phi"]   \arrow[ur,dotted,dash,gray,swap,"\delta"]              &                  & \tilde{\sketchfn} \arrow[r,dotted,dash,gray,swap,"\varepsilon_\text{B}"] & \sketchfn' \arrow[ur,dotted,dash,gray,swap,"\gamma^c(\delta + \varepsilon_\text{R}) + \varepsilon_{\text{E}}"]            
        \end{tikzcd}
    \end{equation*}
    \vspace{-0.7cm}
    \caption{The objects and structure used to analyse the Sketch-DP. 
    }
    \label{fig:convergence-objects}
    \vspace{-0.3cm}
\end{figure}
The second step of the analysis is to chain together the errors that are incurred at each step of dynamic programming, so as to obtain a bound on the asymptotic distance of the sequence $(\sketchfn_k)_{k=0}^\infty$ from $\sketchfn^\pi$, motivated by error propagation analysis in the case of function approximation (\citet{bertsekas1996neuro,munos2003error}; see also \citet{wu2023distributional} in the distributional setting). The next proposition provides the technical tools required for this; the notation is chosen to match the illustration in Figure~\ref{fig:convergence-objects}.

\begin{restatable}{proposition}{propPropagation}(\textbf{Error propagation.})\label{prop:error-prop}
    Consider a norm $\| \cdot \|$ on $\mathbb{R}^m$, and let $\| \cdot \|_\infty$ be the norm on $(\mathbb{R}^{m})^\mathcal{X}$ defined by $\| \sketchfn \|_\infty = \max_{x \in \mathcal{X}} \| \sketchfn(x) \|$. Let $d$ be a metric on RDFs such that $\mathcal{T}^\pi$ is a $\gamma^c$-contraction with respect to $d$ (such as the supermum-Wasserstein and the supermum-Cramer distances). Suppose the following bounds hold.
    \begin{itemize}
        \item (Bellman approximation bound.) For any $\eta \in \mathscr{P}([G_\text{min}, G_{\text{max}}])^\mathcal{X}$,
        \begin{align*}
            \max_{x \in \mathcal{X}} \| \Phi (\mathcal{T}^\pi \eta)(x) - (\mathcal{T}^\pi_\phi \Phi \eta)(x) \| \leq \varepsilon_\text{B} \, .
        \end{align*}
        \item (Reconstruction error bound.) For any $\eta, \bar{\eta} \in \mathscr{P}([G_\text{min}, G_{\text{max}}])^\mathcal{X}$ with sketches $\sketchfn, \bar{\sketchfn}$, we have $d(\eta, \bar{\eta}) \leq \| \sketchfn - \bar{\sketchfn} \|_\infty + \varepsilon_{\text{R}}$.
        \item (Embedding error bound.) For any $\eta', \bar{\eta}' \in \mathscr{P}([G_{\text{min}}, G_{\text{max}}])^\mathcal{X}$ with sketches $\sketchfn', \bar{\sketchfn}'$, we have $\| \sketchfn' - \bar{\sketchfn}' \|_\infty \leq d(\eta', \bar{\eta}') + \varepsilon_{\text{E}}$.
    \end{itemize}
    Then for any two return-distribution functions $\eta, \bar{\eta} \in \mathscr{P}([G_\text{min}, G_{\text{max}}])^{\mathcal{X}}$ with sketches $\sketchfn, \bar{\sketchfn}$ satisfying $\| \sketchfn - \bar{\sketchfn} \| \leq \delta$, we have
    \begin{align*}
        \| \Phi \mathcal{T}^\pi \eta - \mathcal{T}^\pi_\phi \bar{\sketchfn} \|_\infty \leq \gamma^c (\delta + \varepsilon_\text{R}) + \varepsilon_\text{R} + \varepsilon_{\text{E}} \, .
    \end{align*}
\end{restatable}

A formal proof is given in Appendix~\ref{sec:proofs}; Figure~\ref{fig:convergence-objects} (bottom) shows the intuition, propagating bounds through different intermediate stages of the analysis of the update. The main error bound result combines the two earlier results.

\begin{restatable}{proposition}{propBound}\label{prop:bound}
    Suppose the assumptions of Proposition~\ref{prop:error-prop} hold, that $\mathcal{T}^\pi$ maps $\mathscr{P}([G_\text{min}, G_{\text{max}}])^\mathcal{X}$ to itself, and suppose $\mathcal{T}^\pi_\phi$ maps $\{ \Phi \nu : \nu \in \mathscr{P}([G_\text{min}, G_{\text{max}}])^\mathcal{X} \}$ to itself. Then for a sequence of sketches $(\sketchfn_k)_{k=0}^\infty$ defined iteratively via $\sketchfn_{k+1} = \mathcal{T}^\pi_\phi \sketchfn_k$, we have
    \begin{align*}
        \limsup_{k \rightarrow \infty} \| \sketchfn_k - \sketchfn^\pi \| \leq \frac{1}{1-\gamma^c} (\gamma^c \varepsilon_\text{R} + \varepsilon_\text{B} + \varepsilon_{\text{E}}) \, .
    \end{align*}
\end{restatable}
\begin{proof}
    For each $\sketchfn_k$, let $\eta_k$ be an RDF with the property $\Phi \eta_k = \sketchfn_k$. Applying Proposition~\ref{prop:error-prop} to sketches $\sketchfn^\pi$ and $\sketchfn_k$, we obtain
    $
        \| \sketchfn_{k+1} - \sketchfn^\pi \|_\infty \leq \gamma^c \| \sketchfn_k - \sketchfn^\pi \|_\infty + \gamma^c \varepsilon_{\text{R}} + \varepsilon_{\text{B}} + \varepsilon_\text{E} \, .
    $
    Taking a limsup on both sides over $k$ and rearranging yields the result.
\end{proof}

\subsection{Concrete example}

The analysis presented above is abstract; it provides a generic template for conducting error propagation analysis to show that Sketch-DP converges to a neighbourhood of the true values, and moreover illustrates the dependence of this error on the ``richness'' of the sketch, and accuracy of the Bellman coefficients. 
To apply this abstract result to a concrete algorithm, we are required to establish the three error bounds that appear in the statement of Proposition~\ref{prop:error-prop}. The result below shows how this can lead to a concrete result for a novel class of sketches; in particular, proving that computed mean embeddings under these features become arbitrarily accurate as the number of features increases.

\begin{restatable}{proposition}{propIndic}\label{prop:concrete_bound}
    Consider a sketch $\phi$ whose coordinates are feature functions of the form $\phi_i(z) = \mathbbm{1}\{ z_1 \leq z < z_{i+1} \}$ ($i=1,\ldots,m-1$), and $\phi_m(z) = \mathbbm{1}\{ z_1 \leq z \leq z_{m+1} \}$, where $z_1,\ldots,z_{m+1}$ is an equally-spaced grid over $[G_{\text{min}}, G_{\text{max}}]$, with $G_{\text{min}} = \min \mathcal{R} / (1-\gamma)$, $G_{\text{max}} = \max \mathcal{R} / (1-\gamma)$. Let $\mathcal{T}^\pi_\phi$ be the corresponding Sketch-DP operator given by solving \Eqref{eq:regression} with $\mu = \text{Unif}([G_{\text{min}}, G_{\text{max}}])$, and define a sequence $(\sketchfn_k)_{k=0}^\infty$ by taking $\sketchfn_0(x)$ to be the sketch of some initial distribution in $\mathscr{P}([G_{\text{min}}, G_{\text{max}}])$, and $\sketchfn_{k+1} = \mathcal{T}^\pi_\phi \sketchfn_k$ for all $k \geq 0$. Let $\sketchfn^\pi \in (\mathbb{R}^{m})^\mathcal{X}$ be the mean embeddings of the true return distributions. Finally, let $\| \cdot \|$ be the norm on $\mathbb{R}^m$ defined by
    $
        \| u \| = \frac{G_{\text{max}} - G_{\text{min}}}{m} \sum_{i=1}^m |u_i| \, .
    $
    Then we have
    \begin{align*}
        \limsup_{k \rightarrow \infty} \| \sketchfn_k - \sketchfn^\pi \|_\infty \leq \frac{(G_{\text{max}} - G_{\text{min}}) (3 + 2\gamma)}{(1-\gamma) m} \, .
    \end{align*}
\end{restatable}

\begin{figure*}[ht!]
    \centering
    \includegraphics[width=0.493\textwidth]{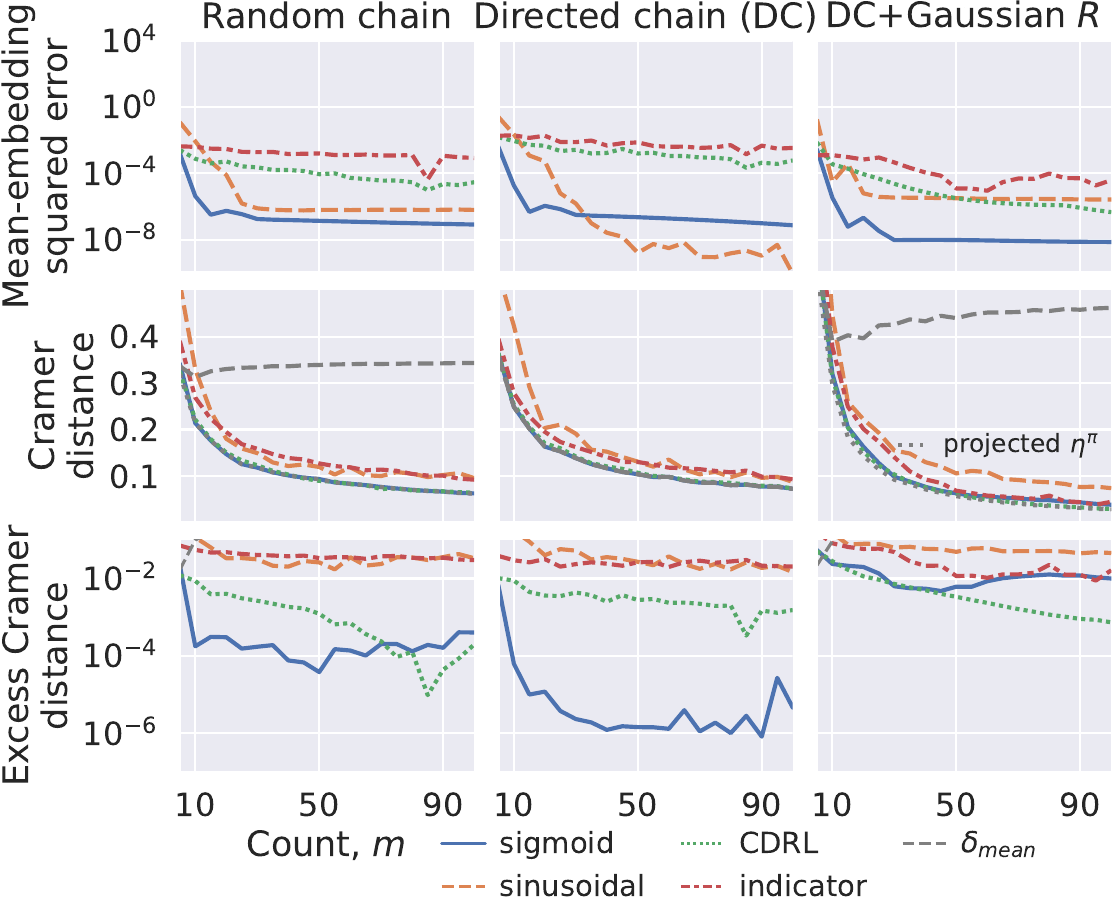}\hfill
    \includegraphics[width=0.473\textwidth,
    trim=0 -0.7cm 0 0, clip=true]{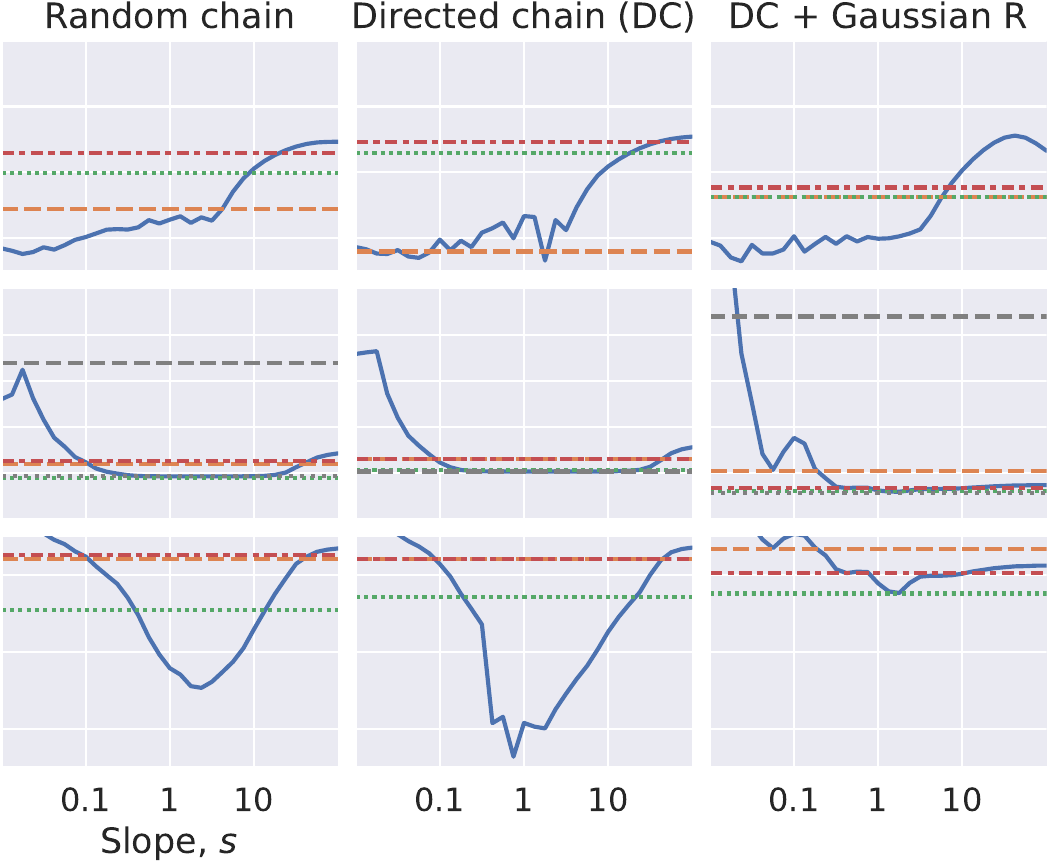}
    \vspace{-0.3cm}
    \caption{Results of running Sketch-DP (\cref{alg:dp}) on tabular environments.
    }
    \vspace{-0.3cm}
    \label{fig:tabular}
\end{figure*}
\section{Experiments}\label{sec:experiments}\label{sec:tabular}
We first test quality of return distribution predictions by the sketch algorithms, investigating the effects of three key factors in \Eqref{eq:translation-family}:
the base feature $\kappa$, 
the number of features $m$,
and the slope $s$, using three tabular MRPs (details in \cref{sec:tabular-details}, extended results in \cref{sec:further-tabular-experiments}).
In \textbf{Random chain}, the transitions are random, rewards are deterministic;
in \textbf{Directed chain} (DC), both transitions and rewards are deterministic;
and in \textbf{DC+Gaussian R}, the transitions are deterministic and the rewards are Gaussian.
We compare the mean embeddings estimated by Sketch-DP with ground-truth mean embeddings, 
reporting their squared $L^2$ distance (\textbf{mean embedding squared error}),
and also compare the \textbf{Cram\'er distance} $\max_{x \in \mathcal{X}}\ell_2^2(\hat{\eta}(x), \eta^\pi(x))$ (see e.g. \citet{rowland2018analysis}) between the distribution $\hat{\eta}(x)$ imputed from the Sketch-DP estimate against the ground-truth $\eta^\pi(x)$.
To aid interpretation of the Cram\'er distance results, we also report the Cram\'er distance between the ground truth $\eta^\pi(x)$ and two baselines. First, the Dirac delta $\delta_{V^\pi(X)}$ at the mean return; we expect Sketch-DP to outperform this na\"ive baseline by better capturing properties of the return distribution beyond the mean. Second, the return distribution estimate computed by categorical DP \citep{rowland2018analysis,bdr2023}, a well-understood approach to distrbutional RL based on categorical distributions.

The results for sweeps over feature count $m$ and slope $s$ are shown in \cref{fig:tabular}. 
By sweeping over $m$, we see that the estimated mean embedding 
goes towards the ground-truth as we use more features. 
Further, the Cram\'er distance also decreases as $m$ increases,
suggesting that the distribution represented also approaches the ground-truth. To
highlight differences between various Sketch-DP algorithms, we also compute the 
\textbf{excess Cram\'er}: the Cram\'er distance 
as above, minus the corresponding distance between 
the categorical projection of $\eta^\pi$ and $\eta^\pi$ itself.
All distributional methods perform well on these tasks, and significantly outperform the Dirac estimator in stochastic environments; we note that all methods have tunable hyperparameters (bin locations for CDRL, feature parameters for Sketch-DP), which should inform direct comparison between methods.
The results of the sweep on the slope parameter $s$ show different trends depending on 
the metric. For smoother $\phi$, generally we can obtain smaller error on the mean embeddings,
but the Cram\'er distances are only small for intermediate range of slope values. This result
is expected: when the features are too smooth or too sharp,
there exists regions within the return range where the feature values do not 
vary meaningfully. This results in a more lossy encoding of the return distribution, indicating the importance of tuning the slope parameter of the translation family (\Eqref{eq:translation-family}).

We include additional experiments in \cref{sec:sfdp_details} showing that Sketch-DP outperforms and is substantially faster in wallclock run-time than SFDP based on imputation strategies \citep{rowland2019statistics, bdr2023}. The faster speed of Sketch-DP is because its updates involve simple linear-algebraic operations, as opposed to the more involved DP update, using imputation strategies, in SFDP.

\begin{figure*}[t]
    \centering
    \null\hfill
    \includegraphics[keepaspectratio,width=.45\textwidth]{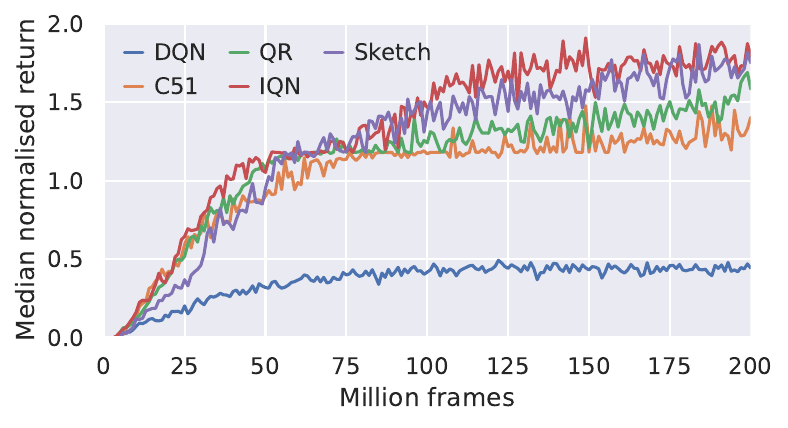}\hfill
    \includegraphics[keepaspectratio,width=.45\textwidth]{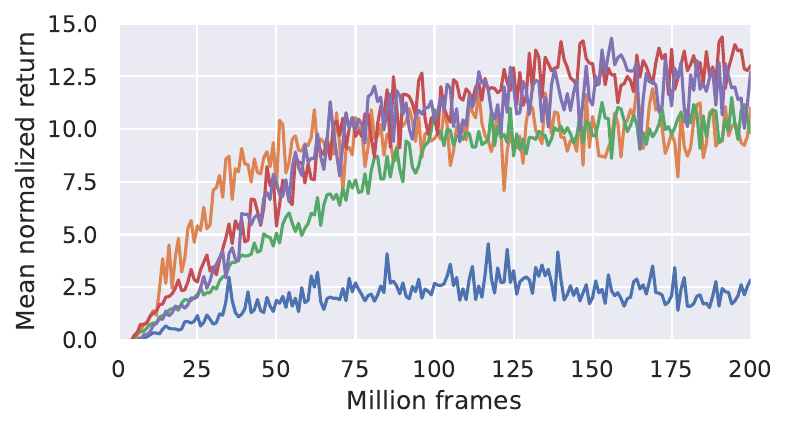}\hfill
    \null
    \vspace{-0.5cm}
    \caption{Median (left) and mean (right) human-normalised scores on the Atari 57 suite.}
    \label{fig:deep-rl}
    \vspace{-0.3cm}
\end{figure*}

\subsection{Deep reinforcement learning}\label{sec:deep-rl}

The primary motivation of our work has been to develop principled novel approaches to distributional RL based on mean embeddings.
Here, we also verify that the Bellman sketch framework is robust enough to apply in combination with deep reinforcement learning. We train neural-network predictions $\sketchfn_\theta(x, a)$ of sketch values for each state-action pair $(x, a)$. To be able to define greedy policy improvements based on estimated sketch values, we precompute 
\emph{value-readout coefficients} 
$\beta \in \mathbb{R}^m$ by solving
$
     \argmin_\beta \mathbb{E}_{G \sim \mu}[ ( G - \langle \beta,  \phi(G) \rangle )^2 ] \, ,
$
so that we can predict expected returns from the sketch value as $\langle \beta, \sketchfn_\theta(x, a)\rangle$. This allows us to define a Q-learning-style update rule: given a transition $(x, a, r, x')$, first compute $a' = \argmax_{\tilde{a}}\langle \beta, \sketchfn_{\bar{\theta}}(x', \tilde{a}) \rangle$, and then the gradient:
$
\nabla_\theta \| \sketchfn_\theta(x, a) - \bellmancoeff_r \sketchfn_{\bar{\theta}}(x', a') \|^2_2 \, ,
$
where $\bar{\theta}$ are the target network parameters. In our experiments, we parametrise $\sketchfn_\theta$ according to the architecture of QR-DQN \citep{dabney2018distributional}, so that the $m$ outputs of the network predict the values of the $m$ coordinates of the corresponding sketch value. We use the sigmoid function as the base feature $\kappa$. Full experimental details for replication are in Appendix~\ref{sec:app-deep-rl}; further results are in \cref{sec:extended-atari}.

Figure~\ref{fig:deep-rl} shows the mean and median human-normalised performance on the Atari suite of environments \citep{bellemare2013arcade} across 200M training frames, and includes comparisons against DQN \citep{mnih2015human}, as well as the distributional agents C51 \citep{bellemare2017distributional}, QR-DQN \citep{dabney2018distributional}, and IQN \citep{dabney2018implicit}. Sketch-DQN attains higher performance on both metrics relative to the comparator agents C51 and QR-DQN, and approaches the performance of IQN, which uses a more complex prediction network to make non-parametric predictions of the quantile function of the return.
In addition, Sketch-DQN runs faster than QR-DQN and IQN; see \cref{sec:runtime_details}.
These results indicate that the sketch framework can be reliably applied to deep RL.
Code is available \href{https://github.com/google-deepmind/sketch_dqn}{here}.

\section{Related work}\label{sec:related-work}

Typical approaches to distributional RL focus on learning approximate distributions directly (see, e.g., \citet{bellemare2017distributional,dabney2018distributional,yang2019fully,nguyen2020distributional,wu2023distributional}). 
Much prior work has considered statistical functionals of the random return, at varying levels of generality with regard to the underlying Markov decision process model. See for example \citet{mandl1971variance,farahmand2019value} for work on characteristic functions, \citet{chung87discounted} for the Laplace transform, \cite{tamar2013temporal,tamar2016learning} for variance, and \citet{sobel1982variance} for higher moments. 
Our use of finite-dimensional mean embeddings is inspired by distributed distributional codes (DDCs)
from theoretical neuroscience \citep{sahani2003doubly,vertes2018flexible,wenliang2019neurally}, which can be regarded as  
neural activities encoding return distributions.
DDCs were previously used to model transition dynamics and successor features in partially observable MDPs \citep{vertes2019neurally}. 
\citet{tano2020local} consider applying non-linearities to rewards themselves, rather than the return, and learning with a variety of discount factors, to encode the distribution of rewards at each timestep.
The sketches in this paper are in fact mean embeddings into finite-dimensional reproducing kernel Hilbert spaces (RKHSs; the kernel corresponding to the feature function $\phi$ is $K(z, z') = \langle \phi(z) , \phi(z') \rangle$).
Kernel mean embeddings
have been used in RL for representing state-transition distributions  \citep{GruLevBalPonetal12,BooGreGeo13,LevShaStaSze16,ChoRaf23}, and 
maximum mean discrepancies (MMDs) \citep{gretton2012kernel} have been used to define losses in distributional RL by \citet{nguyen2020distributional}.
\citet{nguyen2020distributional} combine an MMD loss and distributional 
bootstrapping to define an incremental learning algorithm, 
but it does not naturally lead to a DP formulation, and its convergence is 
not analysed. In contrast, our sketch framework
not only generalises their approach to define both DP and TD algorithms
but also allows rigorous analysis of DP convergence as presented in \cref{sec:convergence}. We elaborate the comparisons to MMDRL and a few other distributional RL methods in \cref{sec:comparison_details}. 

\section{Conclusion}

We have proposed a framework for distributional RL based on Bellman updates that take place entirely within the sketch domain. This has yielded new dynamic programming and temporal-difference learning algorithms and a novel error propagation analysis. We have provided empirical validation on a suite of tabular MRPs and demonstrated that the approach can be successfully applied as a variant of the DQN. While convergence analysis for general sketches is an immediate future work, we expect that there will be benefits from further exploration of algorithmic possibilities opened up by this framework, and potential consequences for modelling value representations in the brain.

\bibliography{main}
\bibliographystyle{iclr2024_conference}

\clearpage

\appendix
\onecolumn
\begin{center}
\large
\textbf{\thetitle:\\Supplementary Material}
\end{center}

\section{Proofs}\label{sec:proofs}

\propRegressionBound*

\begin{proof}
    Let $(G(x) : x \in \mathcal{X})$ be an instantiation of $\eta$ \citep{bdr2023}; that is, a collection of random variables such that for each $x \in \mathcal{X}$, we have $G(x) \sim \eta(x)$. First, note that
    the distribution $(\mathcal{T}^\pi \eta)(x)$ is exactly the distribution of $R + \gamma G(X')$ (when the transition begins at $x$ and is generated by $\pi$). So we have
    \begin{align*}
        \Phi (\mathcal{T}^\pi \eta)(x) = \mathbb{E}_{Z \sim (\mathcal{T}^\pi \eta)(x)}[\phi(Z)] = \mathbb{E}^\pi_x[\phi(R + \gamma G(X'))] \, .
    \end{align*}
    It then follows that:
    \begin{align*}
        \max_{x\in\mathcal{X}}\| \Phi (\mathcal{T}^\pi \eta)(x) - (\mathcal{T}^\pi_\phi \Phi \eta)(x) \| &= \max_{x \in \mathcal{X}} \Big\| \mathbb{E}^\pi_x\Big[\phi(R + \gamma G(X'))\Big] - \mathbb{E}^\pi_x\Big[ \bellmancoeff_R \mathbb{E}[\phi(G(X')) | X'] \Big] \Big\| \\
        & = \max_{x \in \mathcal{X}} \Big \| \mathbb{E}^\pi_x\Big[\phi(R + \gamma G(X')) - \bellmancoeff_R \phi(G(X')) \Big] \Big\| \\
        & \leq \max_{x \in \mathcal{X}} \mathbb{E}^\pi_x\Big[ \big\| \phi(R + \gamma G(X')) - \bellmancoeff_R \phi(G(X')) \big\| \Big] \\
        & \leq \max_{x \in \mathcal{X}} \max_{g \in [G_{\text{min}}, G_{\text{max}}]} \max_{r \in \mathcal{R}} \| \phi(r + \gamma g) - B_r \phi(g) \| \, ,
    \end{align*}
    as required.
\end{proof}

\propPropagation*

\begin{proof}
    We follow the illustration laid out in Figure~\ref{fig:convergence-objects}:
    \begin{align*}
        \| \mathcal{T}^\pi_\phi \sketchfn - \Phi \mathcal{T}^\pi \bar{\eta} \|_\infty 
        & \overset{(a)}{\leq} \| \mathcal{T}^\pi_\phi \sketchfn - \Phi \mathcal{T}^\pi \eta \|_\infty + \| \Phi \mathcal{T}^\pi \eta  - \Phi \mathcal{T}^\pi \bar{\eta} \|_\infty \\
        & \overset{(b)}{\leq} \varepsilon_\text{B} + \| \Phi \mathcal{T}^\pi \eta  - \Phi \mathcal{T}^\pi \bar{\eta} \|_\infty \\
        & \overset{(c)}{\leq} \varepsilon_\text{B} + d(  \mathcal{T}^\pi \eta, \mathcal{T}^\pi \bar{\eta}) + \varepsilon_{\text{E}} \\
        & \overset{(d)}{\leq} \varepsilon_\text{B} + \gamma^c d(  \eta, \bar{\eta}) + \varepsilon_{\text{E}} \\
        & \overset{(e)}{\leq} \varepsilon_\text{B} + \gamma^c (\delta + \varepsilon_\text{R}) + \varepsilon_{\text{E}} \, ,
    \end{align*}
    as required, where (a) follows from the triangle inequality, (b) follows from the Bellman approximation bound, (c) follows from the embedding error bound, (d) follows from $\gamma^c$-contractivity of $\mathcal{T}^\pi$ with respect to $d$, and (e) follows from the reconstruction error bound.
\end{proof}

\propIndic*

\begin{proof}
    We begin by obtaining reconstruction and embedding error bounds for this sketch. We introduce the shorthand $\Delta = (G_{\text{max}} - G_{\text{min}})/m$. 
    To obtain a reconstruction error bound, for any distribution $\nu \in \mathscr{P}([z_1,z_{m+1}])$, define $\Pi \nu$ to be the distribution obtained by mapping each point of support $z$ of $\nu$ to the greatest $z_i$ less than or equal to $z$. Mathematically, if we define $f(z) = \max \{z_i : z_i \leq z \}$, then $\Pi \nu = f_\# \nu$, i.e.\ $\Pi \nu$ is the pushforward of $\nu$ through $f$. We then have $w_1(\nu, \Pi \nu) \leq \Delta$ for all $\nu$ supported on $[z_1, z_m]$, where $w_1$ is the 1-Wasserstein distance, since $f$ transports mass by at most $\Delta$. 
    Introducing another distribution $\nu'$ and the projection $\Pi \nu'$, we note that $w_1(\Pi \nu, \Pi \nu') = \| \Phi \nu - \Phi \nu' \|$. Combining these observations with the triangle inequality yields
    \begin{align*}
        w_1(\nu, \nu') \leq w_1(\nu, \Pi \nu) + \| \Phi \nu - \Phi \nu' \| + w_1(\nu', \Pi \nu') \leq \| \Phi \nu - \Phi \nu' \| + 2 \Delta \, ,
    \end{align*}
    which gives the required form of reconstruction bound, with $\varepsilon_{\text{R}} = 2\Delta$, for the supremum-Wasserstein distance $\overline{w}_1(\eta, \eta') = \max_{x \in \mathcal{X}} w_1(\eta(x), \eta'(x))$ defined over RDFs $\eta, \eta' \in \mathscr{P}(\mathbb{R})^\mathcal{X}$. We can also essentially reverse the argument to get
    \begin{align*}
        \| \Phi \nu - \Phi \nu' \| = w_1(\Pi \nu, \Pi \nu') \leq w_1(\Pi \nu, \nu) + w_1(\nu, \nu') + w_1(\nu', \Pi \nu') \leq w_1(\nu, \nu') + 2 \Delta
    \end{align*}
    which gives the required form of the embedding error bound, with $\varepsilon_{\text{E}} = 2\Delta$.

    Additionally, we can analyse the worst-case regression error $\| \phi(r + \gamma g) - \bellmancoeff_r \phi(g)\|$ to get a bound on the Bellman approximation $\varepsilon_\text{B}$, by Proposition~\ref{prop:one-step}. Observe that $\phi(g)$ is constant for $g \in [z_i, z_{i+1})$, and equal to
    \begin{align*}
        (\underbrace{1, \ldots, 1}_{i \text{ times}}, 0, \ldots, 0)^\top \, .
    \end{align*}
    The minimum regression error in
    \begin{align}\label{eq:min-regression-error}
        \mathbb{E}_{G \sim \text{Unif}([z_1, z_m]} [ \| \phi(r + \gamma G) - \bellmancoeff_r \phi(G) \| ]
    \end{align}
    is therefore obtained by setting the $i$\textsuperscript{th} column of $\bellmancoeff_r$ so that
    \begin{align*}
        \bellmancoeff_r \phi(z_i) = \mathbb{E}_{G \sim \text{Unif}([z_i, z_{i+1}))}[ \phi(r + \gamma G) ] \, ;
    \end{align*}
    note the support of the distribution in the line above. Since $r + \gamma G$ in this expectation varies over an interval of width $\gamma \Delta$, the integrand $\phi(r + \gamma G)$ takes on at most two distinct values. It then follows that we can bound the minimum regression error in \Eqref{eq:min-regression-error} by $\Delta$, and hence we can take $\varepsilon_{\text{B}} = \Delta$.
    
    Finally, we observe that $\mathcal{T}^\pi$ maps $\mathscr{P}([G_{\text{min}}, G_{\text{max}}])$ to itself, since for any $g \in [G_{\text{min}}, G_{\text{max}}]$ and any $r \in \mathcal{R}$, we have by construction of $G_{\text{min}}$, $G_{\text{max}}$ that $r + \gamma g \in [G_{\text{min}}, G_{\text{max}}]$. In addition, we have $\{ \Phi \nu : \nu \in \mathscr{P}([G_{\text{min}}, G_{\text{max}}]_\} = \{ u \in \mathbb{R}^m : 0 \leq u_1 \leq \cdots \leq u_{m-1} \leq u_m = 1 \}$, and by the inspection of the columns of $\bellmancoeff_r$ above, it follows that $\mathcal{T}^\pi_\phi$ maps $\{ \Phi \nu : \nu \in \mathscr{P}([G_{\text{min}}, G_{\text{max}}]_\}^\mathcal{X}$ to itself. Therefore the conclusion of Proposition~\ref{prop:bound} holds, and we obtain 
    \begin{align*}
        \limsup_{k \rightarrow \infty} \| U_k - U^\pi \|_\infty & \leq \frac{1}{1-\gamma}(\gamma \varepsilon_\text{R} + \varepsilon_\text{B} + \varepsilon_\text{E}) \\
        & \leq \frac{1}{1-\gamma} (\gamma 2 \Delta + \Delta + 2 \Delta) \\
        & = \frac{\Delta(3 +2 \gamma)}{1-\gamma} \\
        & = \frac{(G_{\text{max}} - G_{\text{min}})(3 + 2 \gamma)}{(1-\gamma)m}
    \end{align*}
    as required.
\end{proof}

\section{Further details and extensions}\label{sec:further-details-and-extensions}

In this section, we collect further details on a number of topics raised in the main paper.

\subsection{Categorical imputation}\label{sec:cat-imputation}

In the tabular experiments in Sections~\ref{sec:fork-example} and \ref{sec:tabular}, we include comparisons of distributions imputed from the learned mean embeddings, to provide an interpretable comparison between the different Sketch-DP methods studied. Here, we provide a detailed description of the imputation method.

For a given feature map $\phi$, and a learned sketch value $u$, the goal is to define an imputation strategy $\iota : \mathbb{R}^m \rightarrow \mathscr{P}(\mathbb{R})$ \citep{rowland2019statistics,bdr2023}; that is, a function with the property $\mathbb{E}_{Z \sim \iota(u)}[\phi(Z)] \approx u$, so that $\iota$ serves as an approximate pseudo-inverse to the mean embedding. Here, we follow the approach of \citet{song2008tailoring}, and impute probability distributions supported on a finite support set $\{ z_1,\ldots,z_n\}$. We define $\iota(s)$ implicitly through the following (convex) quadratic program
\begin{align*}
    \argmin_{p \in \Delta_n} \Big\| \sum_{i=1}^n p_i \phi(z_i) - s \Big\|_2^2 \, .
\end{align*}
Note that the left-hand term inside is the expectation of $\phi(Z)$ with $Z \sim \sum_{i=1}^n p_i \delta_{z_i}$, and so the objective is simply aiming to minimise the squared error between the learned sketch value and the sketch value from this discrete distribution. Since this quadratic program is convex, it is solvable efficiently; in our implementations, we use SciPy's \textsc{minimize} algorithm \citep{2020SciPy-NMeth}.

\subsection{Computational properties of Bellman coefficients}\label{sec:computational-properties}

Under many choices of feature maps $\phi$, the matrix $\bellmancoeff_r$ has structure that may be exploited computationally. We provide sketches of several cases of interest.
For ``binning features'', even for overlapping bins, $\bellmancoeff_r$ is a very narrow band matrix, and hence is sparse, leading to linear-time matrix-vector product computation.
This remains approximately true for other forms of localised features, such as low-bandwidth Gaussians and related bump-like functions, and in particular applying truncation to near-zero coefficient in the Bellman coefficients in such cases will also lead to sparse matrices.

\textbf{Bellman coefficients as least-squares coefficients.} The closed-form solution for the Bellman coefficients $B_r$ in \Eqref{eq:AB} can be derived by viewing the optimisation problem in \Eqref{eq:regression} as a vector-valued linear regression problem, and using the usual expression for the optimal prediction coefficients. The derivation is the same in content to the usual derivation of least-squares coefficients, which we provide below for completeness, to illustrate how it is obtained in our case. We begin by differentiating the (quadratic) objective in \Eqref{eq:AB} with respect to $B$, and setting the resulting expression equal to the zero vector, to obtain
\begin{align*}
    -2 \mathbb{E}_{G \sim \mu}[\phi(r + \gamma G) \phi(G)^\top] + 2 \mathbb{E}_{G\ \sim \mu}[B_r \phi(G) \phi(G)^\top] = 0 \, .
\end{align*}
Rearranging, we obtain
\begin{align*}
    B_r \mathbb{E}_{G\ \sim \mu}[\phi(G) \phi(G)^\top] = \mathbb{E}_{G \sim \mu}[\phi(r + \gamma G) \phi(G)^\top] \, .
\end{align*}
Finally, under the assumption of invertibility of $\mathbb{E}_{G\ \sim \mu}[\phi(G) \phi(G)^\top]$, we obtain the expression for the Bellman coefficients in \Eqref{eq:AB}:
\begin{align*}
    B_r = \mathbb{E}_{G \sim \mu}[\phi(r + \gamma G) \phi(G)^\top]\mathbb{E}_{G\ \sim \mu}[\phi(G) \phi(G)^\top]^{-1} \, .
\end{align*}

\textbf{Online computation of Bellman coefficients in the case of unknown rewards.} In settings where the set of possible rewards $\mathcal{R}$ is not known in advance, is infinite, or is too large to cache Bellman coefficients for all possible rewards $r \in \mathcal{R}$, we may exploit the structure of the Bellman coefficients described above to speed up the computation of the coefficients online. Rather than solving the regression problem from scratch, an alternative is to cache the matrix $\mathbb{E}_{G\ \sim \mu}[\phi(G) \phi(G)^\top]^{-1}$ above, and construct the matrix $\mathbb{E}_{G \sim \mu}[\phi(r + \gamma G) \phi(G)^\top]$ as required, upon observing a new reward $r$. This reduces the marginal cost of computing the Bellman coefficients $B_r$ to a matrix-matrix product.

\subsection{Choices of regression distribution $\mu$}\label{sec:regression-details}

In the main paper, we note that the one-dimensional integrals defining the matrices $C$ and $C_r$ which in turn define the Bellman coefficients $\bellmancoeff_r$ can be computed in a variety of ways, depending on the choice of $\mu$ and feature map $\phi$. In our experiments, we take $\nu$ to be a finitely-supported grid in the range $[\hat{G}_{\text{min}} - b\hat{L}, \hat{G}_{\text{max}} + b\hat{L}]$, where $b$ is casually chosen to be around 0.2. The support of $\nu$ is thus slightly 
wider than the estimated return range and slightly narrower than the anchor range described in Section~\ref{sec:feature-type-details}. 
The intuition for using a wider anchor range is that we need the features to cover the return distribution (and $\nu$) with the non-trivial support of the 
features. We validate this intuition in an additional experiment in \cref{sec:further-tabular-experiments}.
With this $\nu$, \Eqref{eq:regression} is a standard regression problem, and $C$ and $C_r$ can be computed with standard linear-algebraic operations.

Another possibility, particularly if one wishes to use $\mu$ which is not finitely supported, is to use numerical integration to compute these integrals. Additionally, in certain settings the integrals may be computed analytically. For example, 
with Gaussian $\phi_i(x)=\exp(-s^2(x-z_i)^2/2)$ 
and Gaussian $\mu$, or $\mu$ as Lebesgue measure (in which case, technically, we modify the expectation in \Eqref{eq:regression} into an integral against an unnormalised measure), $C$ and $C_r$ can be computed analytically.
In the case of $\mu$ as Lebesgue measure, we have
\begin{align*}
  C_{ij} ~=~ \sqrt{\frac\pi{2s}}~\exp\left(-\frac{s}2 (z_i-z_j)^2\right) ~\text{and}~
  (C_r)_{ij} ~=~ \sqrt{\frac\pi{s(1+\gamma^2)}}~\exp\left(-\frac{s(r+\gamma z_i-\gamma z_j)^2)}{1+\gamma^2}\right) \, .
\end{align*}

\subsection{Knowledge of rewards}\label{sec:knowledge_of_rewards}

In distributional approaches to dynamic programming, it is necessary to know all aspects of the environment's transition structure and reward structure in advance, including the set $\mathcal{R}$ required for precomputing the Bellman coefficients. However, in temporal-difference learning, this is a non-trivial assumption. In many environments, this information is available in advance (in the Atari suite with standard reward clipping post-processing \citep{mnih2015human}, rewards are known to lie in $\{-1,0,1\}$, for example). When this information is not available, one may modify Algorithm~\ref{alg:dp} to instead compute Bellman coefficients for observed rewards \emph{just-in-time}; that is, when these rewards are encountered in a transition. This makes the algorithm more broadly applicable, but clearly incurs a significant cost of computing Bellman coefficients for rewards for which these coefficients are not already cached. As Remark~\ref{remark:linear-regression}, one possibility in this setting is to learn an approximator $H : \mathbb{R} \rightarrow \mathbb{R}^{m \times m}$ that maps from rewards to Bellman coefficients, and use the predictions of the approximator as proxies for the true Bellman coefficients to reduce the need to solve for the Bellman coefficients every time a new reward is encountered.

\subsection{Mathematical properties of Bellman coefficients}\label{sec:bellman-coeffs-details}
The Bellman coefficients $B_r$ play a crucial role in our Bellman sketch framework. Here, we present various properties 
of $B_r$ in a worked example, derived from both a sigmoid and a Gaussian 
base feature in \cref{fig:bellman-coeffs}. For each
base feature, we choose 20 evenly spaced anchors in $[-8, 8]$,
and find $B_r$ for $r=1$ and $\gamma=0.8$, 
and $\mu$ uniformly supported on a dense grid of 10,000 evenly spaced points in $[-5, 5]$.
We apply a small $L^2$ regulariser with weight $10^{-6}$ in the regression problem.

\begin{figure}[t]
    \centering
    \includegraphics[width=\textwidth]{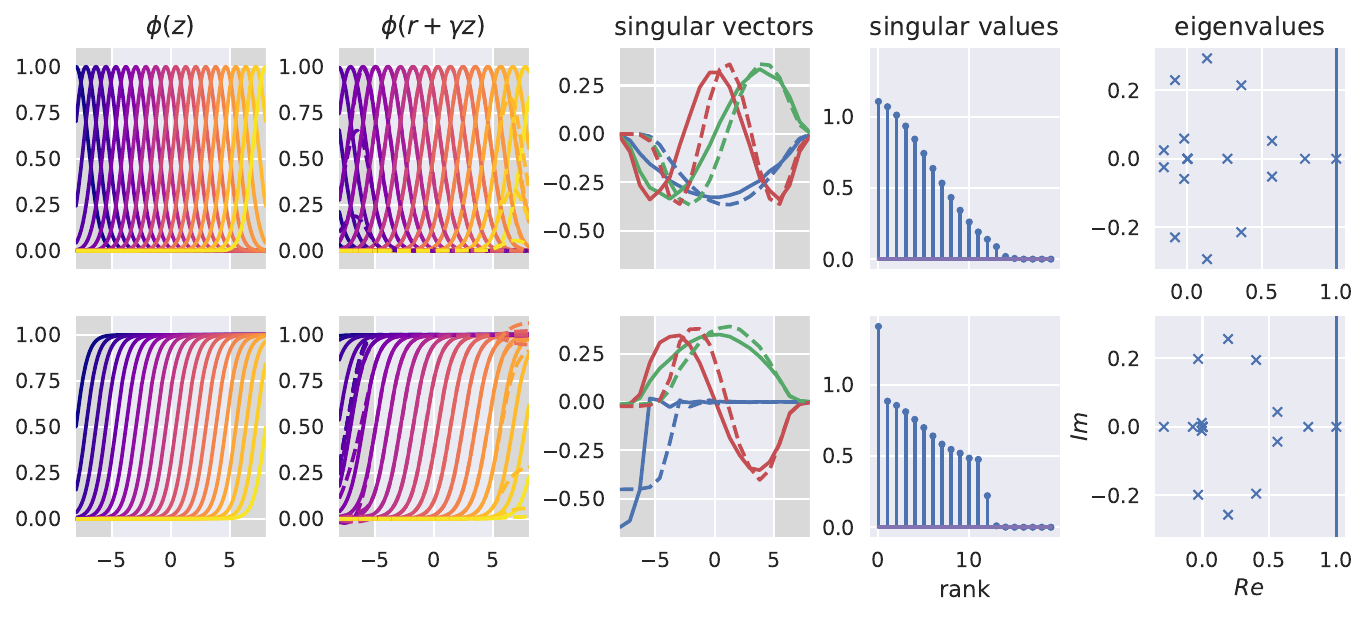}
    \caption{In-depth analysis of Bellman coefficients in the setting described in Section~\ref{sec:bellman-coeffs-details}. In the third column, solid curves are the 
    most significant input/right singular vectors, and dashed lines with matching 
    colours are the corresponding output/left singular vectors.}
    \label{fig:bellman-coeffs}
\end{figure}

First, we assess how accurate approximation in \Eqref{eq:sketch-derivation} is when $\bellmancoeff_r$ is found via the regression problem in \Eqref{eq:regression}.  
In the left two columns of \cref{fig:bellman-coeffs},
we show the feature functions 
$\phi(z)$ and $\phi(r+\gamma z)$ in the first two columns.
In the second column, we also show $B_r\phi(z)$ in dashed lines evaluated on $[-8, 8]$, wider than the 
grid over which we minimised the error. The error is tiny 
and virtually invisible within the interval $[-5,5]$, but is larger 
outside. 
Quantitatively, the maximum absolute difference between 
$\phi(r+\gamma z)$ 
and 
$B_r \phi(z)$
over the dense grid in $[-5, 5]$ is less than 0.002 for both base features. By \cref{prop:one-step}, we 
expect a small error in a single step of dynamical programming.

Had $B_r$ been a contraction, we 
would have been able to prove contraction for \cref{alg:dp}. 
However, we show empirically that $B_r$ is not in general 
a contraction in $L^2$ norm, but the dynamics from repeated 
multiplication of $B_r$ may converge to a stable fixed point.
First, we performed a singular value decomposition of the $B_r$ 
for the two base features. We see that the singular 
vectors in \cref{fig:bellman-coeffs} (third column) are similar to harmonic functions. Importantly, in the largest singular values (operator norms) in \cref{fig:bellman-coeffs} (fourth column) are greater than 1, suggesting that a single application of $B_r$ may
expand the input. Further, we show the eigenvalues 
of $B_r$ in the fifth column of \cref{fig:bellman-coeffs}.
Interestingly, all eigenvalues have real parts less or very 
close to 1.0 suggesting that there exists fixed points in 
the dynamics induced by $B_r$. As such, the Bellman coefficients 
$B_r$ exhibit transient dynamics (typical for 
non-normal matrices) but is stable after repeated 
applications to an initial vector. Further studies 
into these dynamical properties are important for future work.
Given the important role of non-normal dynamics hypothesised to be present in the nervous system
\citep{hennequin2012non,bondanelli2020coding}, 
these observations allude to the possibility that the Bellman sketch framework could contribute to a 
biological implementation of distributional RL.

\subsection{Choices of feature function}\label{sec:feature-type-details}

In the main paper, we note that any set of features spanning the degree-$m$ polynomials is Bellman closed, as described by \citet{rowland2019statistics}, and hence exact dynamic programming is possible with this feature set, as shown by \citet{sobel1982variance}. However, in preliminary
experiments we found these features difficult to learn with temporal-difference methods beyond small values of $m$, due to the widely varying magnitude of moments as $m$ grows, making learning rate selection problematic in stochastic environments; further details are provided in Appendix~\ref{sec:polynomial_tabular}.

In this paper, we tested the following functions as the base feature $\kappa$ in the translation family \Eqref{eq:translation-family}:
\begin{itemize}
    \item Sigmoid: $\kappa(x)=\frac{1}{1 + \exp(-x)}$;
    \item Gaussian: $\kappa(x)=\exp(-x^2/2)$;
    \item Parabolic: $\kappa(x) = 1 - x^2$ for $|x| \le 1$, zero otherwise.
    \item Hyperbolic tangent: $\kappa(x) = \tanh(x)$
\end{itemize}
In Appendix~\ref{sec:further-tabular-experiments}, we provide further experimental results for a wide variety of feature maps from this family. In addition, we also consider the indicator
feature used in \cref{prop:concrete_bound} as well as sinusoidal features shown in \cref{fig:walk_through}B.

We found that the anchor points must be chosen carefully so that the features produce 
variations within the return range. This can be done by choosing the range of the 
anchors to be slightly wider than and estimated return range, and setting the slope so that 
there does not exist a region of the return range that produce no change in the feature functions. In the tabular experiments, we set 
the extremum anchor points to be $\hat{G}_{\text{min}} - a\hat{L}$ and 
$\hat{G}_{\text{max}} + a\hat{L}$, where $\hat{L}=\hat{G}_{\text{min}} - \hat{G}_{\text{min}}$ is the estimated return range,
and $a$ is a small positive value around 0.4 casually chosen and not tuned.
The return limits $\hat{G}_{\text{min}}$ and $\hat{G}_{\text{max}}$ are estimated by the sample minimum and maximum from samples collected by first-visit Monte Carlo; see \cref{sec:tabular-details}.

The slope parameter should depend on the feature and the return range, and we applied the 
following intuition. 
Most of the base feature $\kappa$ have an ``non-trivial support'' that produces
the most variations in the function value. 
For example, the ``non-trivial'' support for the sigmoidal and Gaussian $\kappa$ 
can be chosen as $[-2, 2]$; and for base features that are nonzero only in $[-1,1]$,
this support is $[-1,1]$. We define the width $w$ of a base function as the length of the 
non-trivial support. 
Crudely, the feature with slope $s$ has width $w / s$, as sharper features 
tend to have shorter non-trivial support. 
In addition, for the set of features to cover return range uniformly, 
we set each adjacent feature functions to overlap by 50\%. 
Finally, we want the union of the non-trivial supports of 10 (arbitrary chosen) such 
overlapping features to 
equal the return range. We must then have 
$0.5 \times 10w/s = \hat{G}_{\text{min}}-\hat{G}_{\text{max}}$. 
This is the \emph{default slope} for each feature and each environment with known return range.
As such, the sigmoidal and Gaussian base features have default slope equal to 
$s = 20/(\hat{G}_{\text{min}}-\hat{G}_{\text{max}})$.

\subsection{Comparison with other approaches to distributional RL}
\label{sec:comparison_details}

In this section, we provide additional comparisons against existing approaches to distributional RL. As distributional RL is a quickly evolving field, we focus our comparison on a few main classes of algorithms related to our work, which illustrate some key axes of variation within the field:
(i) categorical approaches \citep{bellemare2017distributional};
(ii) quantile approaches \citep{dabney2018implicit,dabney2018distributional,yang2019fully}; 
(iii) approaches related to maximum mean discrepancy \citep[MMD; ][]{gretton2012kernel}, such as \citet{nguyen2020distributional,zhang2021distributional,sun2022distributional}; and
(iv) sketch-based approaches \citep{sobel1982variance,rowland2019statistics}.

\textbf{Distribution representation.}
Categorical, quantile, and MMD approaches are typically presented as learning approximate return distributions directly. In categorical approaches, the approximate distribution is parametrised as
\begin{align*}
    \sum_{i=1}^m p_i \delta_{z_i} \, ,
\end{align*}
with fixed particle locations $(z_i)_{i=1}^m$, and learnable probabilities $(p_i)_{i=1}^m$ for each state-action pair at which the return distribution is to be approximated. In contrast, quantile and MMD approaches learn fixed-weight particle approximations, of the form
\begin{align}\label{eq:fixed-weight}
    \sum_{i=1}^m \frac{1}{m} \delta_{z_i} \, ,
\end{align}
in which the particle locations $(z_i)_{i=1}^m$ are learnable. Work on sketches has instead focused on learning the values of particular statistical functionals of the return, rather than explicitly approximating return distributions.

\citet{rowland2019statistics} also shows that standard categorical- and quantile-based algorithms can also be viewed through the lens of sketch-based distributional RL. Other work in this vein includes \citet{sobel1982variance}, who analysed the case of moments specifically, and \citet{marthe2023beyond}, who extend the work of \citet{rowland2019statistics} to the undiscounted, finite-horizon case.
The approach proposed in this paper sits firmly in the camp of sketch-based approaches, without ever representing approximated distributions directly.
We highlight generative models of distributions as another prominent class of (non-parametric) representation (see, e.g., \citealt{doan2018gan,freirich2019distributional,dabney2018implicit,yang2019fully,wu2023distributional}).

\textbf{Algorithm types.}
Most prior algorithmic contributions to distributional reinforcement learning have focused on sample-based temporal-difference approaches, in which prediction parameters are iteratively and incrementally updated based on the gradient of a sampled approximation to a loss function. These approaches include the original C51 \citep{bellemare2017distributional}, QR-DQN \citep{dabney2018distributional}, MMDRL \citep{nguyen2020distributional}, and EDRL \citep{rowland2019statistics} algorithms. Dynamic programming algorithms, in which parameters are not updated incrementally via loss gradients, but instead according to the application of an implementable operator, have also been considered (see \citet{rowland2018analysis,rowland2023analysis,wu2023distributional} for categorical dynamic programming, quantile dynamic programming, and fitted likelihood estimation, respectively). 
In this paper, our algorithmic contributions include both DP and TD methods.

\textbf{Losses, projections, and convergence theory.}
One of the core axes of variation across distributional RL approaches is the loss used to define updates in incremental algorithms, and to define projections in dynamic programming. Categorical approaches use a projection in Cram\'er metric \citep{rowland2018analysis} to define a target distribution for both dynamic programming and incremental versions of the algorithm; the incremental algorithm updates predictions via the gradient of a KL loss between the current and target distributions. Quantile-based approaches use either the quantile regression loss in incremental settings, or a Wasserstein-1 projection in dynamic programming \citep{dabney2018distributional}.

The approach proposed in this paper works entirely with mean embeddings of probability distributions. Earlier approaches to sketched-based distributional RL, both in dynamic programming and incremental forms, have defined losses via imputation strategies, which compute updates by converting sketches into approximate distributions \citep{rowland2019statistics,bdr2023}.

Here, we contribute a novel perspective on the work of \citet{nguyen2020distributional}, who propose a sample-based TD algorithm for updating particle locations (as in \Eqref{eq:fixed-weight}) by using an MMD loss, specifically taking the form
\begin{align}\label{eq:mmd}
    \text{MMD}_K^2\Bigg(\sum_{i=1}^m \frac{1}{m} \delta_{z_i(x,a)}, \sum_{i=1}^m \frac{1}{m} \delta_{r + \gamma z_i(x',a')}\Bigg) \, ,
\end{align}
for some choice of kernel $K$. Although not described in this manner by \citet{nguyen2020distributional}, this can be seen as an incremental update on approximate mean embeddings for the RKHS $\mathcal{H}_K$ corresponding to the kernel $K$, amongst the class of mean embeddings of the form
\begin{align}\label{eq:mmdrl-embeddings}
    \sum_{i=1}^m \frac{1}{m} K(z_i, \cdot)  \in \mathcal{H}_K \, ,
\end{align}
where the particle locations $(z_i)_{i=1}^m$ are optimised.
We note that \citet{nguyen2020distributional} do not provide theoretical analysis for their algorithm. They provide contraction analysis of $\mathcal{T}^\pi$ (Theorem~2), and MMD approximation bounds for fixed target distributions (Theorem~3 and Proposition~2), but these do not constitute a proof of convergence of the incremental algorithm described therein. A key reason why the proposed algorithm may not yield clean convergence theory is that the space of mean embeddings described in \Eqref{eq:mmdrl-embeddings}, where only the particles $(z_i)_{i=1}^m$ can vary, is a non-convex subset of an infinite-dimensional RKHS. As such, tractable global optimisation of the objective may not be possible, along with the definition of a straightforward dynamic programming version of this approach. In other words, it is not straightforward to define 
a dynamic programming method to optimise the particle locations 
$(z_i)_{i=1}^m$.

In contrast, the Sketch-DP and Sketch-TD algorithms introduced in this paper work with finite-dimensional RKHS, and define updates via matrix-vector products with the Bellman coefficients derived in \Eqref{eq:regression}. This naturally yields tractable dynamic programming and temporal-difference learning algorithms, and also allows us to develop convergence theory, as described in Section~\ref{sec:convergence}.

Contrasting against the Sketch-DP/Sketch-TD approaches described above, earlier approaches to sketched-based distributional RL, both in dynamic programming and incremental forms, have defined losses via imputation strategies, which compute updates by converting sketches into approximate distributions \citep{rowland2019statistics,bdr2023}. Foreshadowing the remarks on theoretical analysis below, we remark that neither MMDRL nor the earlier sketch-based approach described above have been analysed for convergence, while Section~\ref{sec:convergence} in this paper deals with convergence analysis of the approach proposed in this paper.

In general, convergence analysis of dynamic programming algorithms has been obtained for several classes of distributional algorithms beyond the theory described in this paper; see \citet{rowland2018analysis} for the case of categorical dynamic programming, \citet{dabney2018distributional} for the case of a quantile dynamic programming algorithm, \citet{bdr2023,rowland2023analysis} for later generalisations of this work, and \citet{wu2023distributional} in the case of fitted likelihood evaluation. This analysis typically centres around (i) proving contractivity of the distributional Bellman operator $\mathcal{T}^\pi$ with respect to some metric $\texttt{d}$, and proving non-contractivity of the specific distributional projection used by the dynamic programming algorithm under this same metric. 
Notably, the metric $\texttt{d}$ used in the \emph{analysis} need not be the same as any metrics used in \emph{defining} the algorithm; this is the case for quantile dynamic programming, for which Wasserstein-1 distance is used to define the algorithm, while Wasserstein-$\infty$ distance is used to analyse the algorithm \citep{dabney2018distributional,bdr2023,rowland2023analysis}. Our proof technique in this paper, in particular, makes use of contraction of the distributional Bellman operator in Wasserstein distances, though such distances do not feature in the definition of the Sketch-DP/TD algorithms.

In general, there has been less work on the convergence analysis of sample-based incremental algorithms. \citet{rowland2023analysis} recently showed convergence of quantile temporal-difference learning, though the question of convergence for many other sample-based incremental distributional reinforcement learning algorithms is currently open. The analysis of incremental algorithms is generally more mathematically involved than in the dynamic programming case, principally owing to the fact that rather than analysing the iterated application of a fixed operator, one needs to analyse the continuous dynamical system associated with incremental updates.

\section{Experimental details}\label{sec:experiment-details}

In this section, we provide additional details on the experimental results reported in the main paper.

\subsection{Tabular environments}\label{sec:tabular-details}

We describe the setup in the main paper. In \cref{sec:further-tabular-experiments}, we show extended results of more features and more environments.

\textbf{Environments. }
In the main paper, we reported results on the on the following environments,
\begin{itemize}
    \item \textbf{Random chain:} Ten states $\{x_1, x_2, \ldots, x_{10}\}$ are arranged in a chain. There is equal probability of transitioning to either neighbour at each state, and state $x_{10}$ has a deterministic reward of +1; 
$$
    \mathrm{terminal}\leftarrow x_1 \longleftrightarrow x_2 \longleftrightarrow x_3 \longleftrightarrow \cdots \longleftrightarrow x_{10} \rightarrow \mathrm{terminal} \, .
$$

\item \textbf{Directed chain (DC):} Five states are arranged in a directed chain, 
but the agent can only move along the arrow deterministically until termination. A deterministic reward 
of $+1$ is given at state $x_5$; 
$$
    x_1 \longrightarrow x_2 \longrightarrow x_3 \longrightarrow \cdots \longrightarrow x_{5} \rightarrow \mathrm{terminal} \, .
$$
    \item \textbf{DC with Gaussian reward:} A variant of the directed chain above, with the only difference that $x_5$ has a 
    Gaussian reward with mean 1 and unit variance. 
\end{itemize}
The discount factor is $\gamma=0.9$. These environments cover stochastic and deterministic rewards and state transitions, giving a range of different types of return distributions.

\textbf{Feature functions.} We use features of translation family \Eqref{eq:translation-family}, with $\kappa$ chosen from a subset of base features described in \eqref{sec:feature-type-details}. We also
include the indicator features used in \cref{prop:concrete_bound}. For the sweep
over slope, we set the slope $s$ to be the default slope (described in \cref{sec:feature-type-details}) multiplied by a scaling factor, and sweep over this factor. 
from 0.001 to 10.0. This is done primary because the return range varies a lot across different environments, and the default slope is adjusted to the return range. The results serve as justification for the heuristics on choosing the default slope.

\textbf{Ground-truth distribution.} 
We approximate the ground-truth mean embeddings and the ground-truth return distributions by collecting a large number of return samples from the MRPs. 
To do so, we use first-visit Monte Carlo with 
a sufficiently long horizon (after the first visit to each state) to ensure that the 
samples are unbiased and has bounded error caused by truncating the rollout
to a finite horizon.
For environments with deterministic rewards, truncating the horizon at $L$ steps
induces maximum truncation error $|r|_{\mathrm{max}} \gamma^L / (1 - \gamma)$, where $|r|_\mathrm{max}$ is the maximum reward magnitude. We bound this error at $10^{-4}$,
giving $L>110$, so we set the horizon after the first visit to 110. For
environments with Gaussian rewards, we set the horizon to 200. 
We initialise the rollout at each state in the environment, and for initial each state
this is repeated $10^5$ times. This gives us at least $10^5$ samples each state. 

\textbf{Sketch DP under conditional independence.}
Many RL environments, including the tabular environments tested in this paper,  
have the property that $R \ind X' | X$ for the trajectory 
$X, A, R, X'$, so the Sketch-DP update \Eqref{eq:dp} simplifies to
$$
    \sketchfn(x) \leftarrow \mathbb{E}_x^\pi[\bellmancoeff_R]\mathbb{E}_x^\pi[\sketchfn(X')] = \mathbb{E}_x^\pi[\bellmancoeff_R]\sum_{x'\in\mathcal{X}}P(x'|x)\sketchfn(x').
$$
This means we need to evaluate the expected Bellman coefficient $\mathbb{E}_x^\pi[\bellmancoeff_R]$. 
This is trivial for deterministic rewards. For stochastic rewards with known distributions, we approximate the expectation via numerical integration. We run all DP methods for 200 iterations.

\textbf{Jittered imputation support.}
The support on which we impute the distribution are the anchors of the features.
Some tabular environments have states with deterministic returns that directly align with the feature anchors, which interferes in unintuitive ways with the finite support on which we impute distributions, producing non-monotonic trends in the results. To avoid this unnecessary complication,
we jitter the support before imputing the distribution: for points in the support, we add noise uniformly distributed over $[-\Delta/2, \Delta/2]$, where $\Delta$ is the distance between consecutive support points. Likewise, 
we project ground truth distribution using the same jittered support to. In \cref{fig:tabular}, we 
report the average of the metrics computed from 100 independent jitters. Note that since the imputation (from mean embedding) 
and projection (from ground-truth) share the same support for each of the 100 jitters, the average Cram\'er distance
between the projected and the ground-truth still lower-bounds the average Cram\'er distance between the imputed 
distribution and the ground-truth.

\subsection{Deep reinforcement learning implementation details}\label{sec:app-deep-rl}

In this section, we provide further details on the deep reinforcement learning experiments described in the main paper, in particular describing hyperparameters and relevant sweeps.

\textbf{Environment.} We used the exact same Atari suite environment for benchmarking 
QR-DQN \citep{bellemare2013arcade,dabney2018distributional}. In all experiments, we run three random seeds per environment.

\textbf{Feature map $\phi$.} The results in \cref{fig:deep-rl} uses the sigmoid base feature $\kappa(x)=1 / (e^{-x}+1)$ 
with slope $s=5$, and the anchors to 
be 401 (tuned from 101, 201 and 401) evenly spaced points between $-12$ and $12$. 
These values are loosely motivated
by the range used in C51 \citep{bellemare2017distributional}. 
We did not use the heuristic in \cref{sec:feature-type-details} to choose the slope parameter,
instead we set this to 10 by tuning from $\{1,2,\ldots,12\}$. 
Larger slope values typically resulted in Bellman coefficients with a worst-case regression error
$\max_{r \in \mathcal{R}} \max_{g \in \text{supp}(\mu)} \| \phi(r + \gamma g) - \bellmancoeff_r \phi(g) \|$
greater than 0.01. In these cases, we regard the regression error as too large, and did not perform agent training with these hyperparameter settings. 
In addition, $\phi$ is appended with a constant feature of ones, which we found to be 
very crucial for a good performance; as noted in the main paper, this ensures that the sketch operator is truly affine, not linear. 
We also tried a several other feature functions, including the Gaussian, the hyperbolic tangent and the (Gaussian) error function, and found that the sigmoid reliably performed the best.

\textbf{Solving the regression problems.} To compute the Bellman coefficients
as well as the value readout weights $\beta$ described in \cref{sec:deep-rl}, 
we solve the corresponding regression problem with $\mu$ set to be
100,000 points evenly spaced between $-10$ and $10$. We also add a $L^2$ regulariser
with strength set to $10^{-9}$ to avoid numerical issues, which is tuned from $\{10^{-15}, 10^{-12}, 10^{-9}, 10^{-6}, 10^{-3}\}$. 

\textbf{Neural network.}
We implement Sketch-DQN based on the QR-DQN architecture, using exactly the same convolutional torso and fully-connected layers to estimate the mean embeddings. The differences are:
\begin{itemize}
    \item We add a sigmoid or tanh nonlinearity, depending on the base feature output range to the final layer. This helps bound the predicted mean embedding and improved the results.
    \item The network only predicts the non-constant dimensions of the mean embedding, and the constant feature is appended as a hard-coded value.
    \item We use the pre-computed mean readout coefficients $\beta$ to predict state-action values for state-action pairs in the Q-learning objective, and at current states to determine the greedy policy.
\end{itemize}

\textbf{Training.} We use the exact same training procedure as QR-DQN \citep{dabney2018distributional}. Notably, the learning rate, 
exploration schedule, buffer design are all the same. We tried a small hyperparameter sweep on the learning rate, and found the default learning rate $0.00005$ to be optimal for performance taken at 200 million frames.

\textbf{Evaluation.} The returns are normalised against random and human performance, as reported by \citet{mnih2015human}. We use the mean and median over all games and three random seeds for each game. 

\textbf{Baseline methods}. We also tuned the number of atoms of the approximating distributions in the baseline methods.
In particular, we found that C51 performed the best compared to using more atoms;
IQN did best when using 51 quantiles; and QR-DQN did best using 201 quantiles. Increasing the number of atoms in these methods lead to worse performance. We report 
the results of these best variants of the corresponding baseline methods in \cref{fig:deep-rl}.

\section{Further experiments}\label{sec:further-experiments}

In this section, we collect further experimental results to complement those reported in the main paper.

\subsection{Temporal-difference learning with polynomial features}\label{sec:polynomial_tabular}

As noted in the main text, the sketch corresponding to the polynomial feature function
$\phi(g)=(1, g, g^2, \ldots g^m)$ is Bellman closed, and the Bellman coefficients $\bellmancoeff_r$ obtain zero regression error in \Eqref{eq:regression}. In addition, there has been much prior work on dynamic programming \citep{sobel1982variance} for moments of the return, and temporal-difference learning specifically in the case of the first two moments \citep{tamar2013temporal,tamar2016learning}.
However, such polynomial feature functions are difficult to use as the basis of learning high-dimensional feature embeddings. This stems from several factors, including that the typical scales of the coordinates of the feature function often vary over many orders of magnitude, making tuning of learning rates difficult, as well as the fact that polynomial features are \emph{non-local}, making it more difficult to decode distributional information via an imputation strategy.

To quantify these informal ideas, we ran an experiment comparing Sketch-TD updates for a 50-dimensional mean embedding based on the translation family (\Eqref{eq:translation-family}) with a sigmoid base feature $\kappa$, as well as polynomial features with $m=5$, and $m=50$.
The sigmoid features are chosen according to the intuitions in \cref{sec:feature-type-details}.
We ran 100,000 synchronous TD updates on mean embedding estimates initialised at $\phi(0)$
for all state. Each run uses a fixed learning rate chosen from $10^{-6}$ to $1$.

In Figure~\ref{fig:polynomial}, we plot the Cram\'er distance of imputed distributions from ground-truth after running TD for each of these three methods, on a variety of the environments described in Section~\ref{sec:further-tabular-experiments}. In all environments, there is a similar pattern. For the Sketch-TD algorithm based on sigmoid non-linearities, there is a reasonably wide basin of good learning rates, with performance degrading as the learning rate becomes too small or too large. On several environments this pattern is reflected also in the performance of the degree-5 polynomial embedding, though the minimal Cram\'er error is generally significantly worse than that of the sigmoid embedding. This supports our earlier observations; this mean embedding captures relatively coarse information about the return distribution, and in addition different feature components have different magnitudes, meaning that a constant learning rate cannot perform well. The mean embedding with degree-50 polynomials generally performs very badly on all environments, as the components of the embedding are at such different magnitudes that no appropriate learning rate exists.

\begin{figure}
    \centering
    \includegraphics[width=\textwidth]{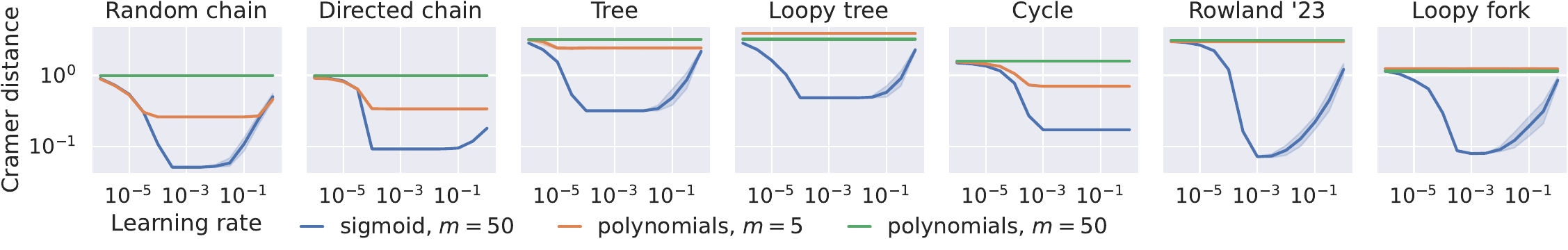}
    \caption{Results comparing Sketch-TD with sigmoidal and polynomial features.}
    \label{fig:polynomial}
\end{figure}

\subsection{Sketch DP example with sigmoidal features}\label{sec:walk_through_details}

\begin{figure*}[t]
    \centering
    \begin{minipage}{0.18\textwidth}
    \footnotesize
    \textbf{A}\\
    \includegraphics[width=0.9\textwidth, trim=-2cm -1.5cm 0 0, clip=true]{img/tabular/DRL_walk_through.pdf}
    \vspace{-0.3cm}\\
    \textbf{B}\\
    \includegraphics[width=0.9\textwidth]{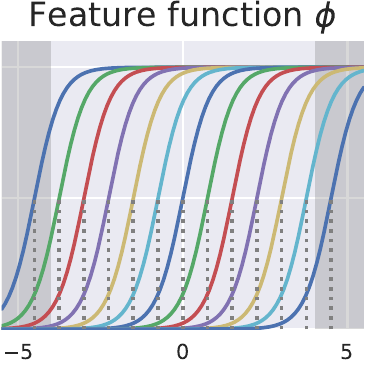}
    \end{minipage}
    \begin{minipage}{0.31\textwidth}
        \textbf{C}\\
        \includegraphics[width=0.9\columnwidth]{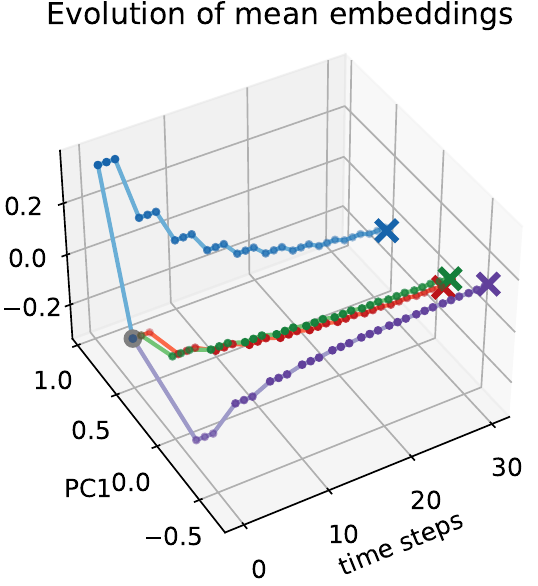}
    \end{minipage}
    \hspace{0.3cm}
    \begin{minipage}{0.47\textwidth} 
        \footnotesize
        \hspace{-0.4cm}\textbf{D} \vspace{-1.2mm}\hspace{2.7cm}\textbf{E}\\
        \includegraphics[width=0.3\textwidth]{img/tabular/walk_through_gt_teal_fourier.pdf}~~~~~~
        \includegraphics[width=0.3\textwidth]{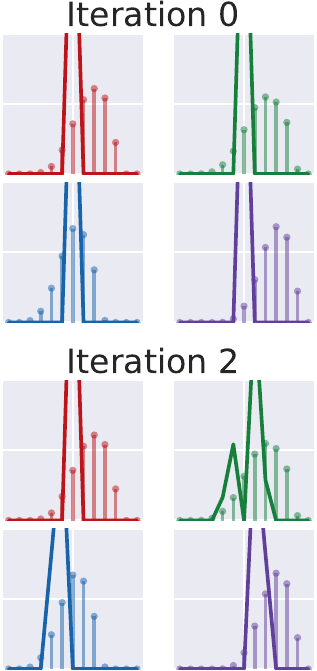}~~~
        \includegraphics[width=0.3\textwidth]{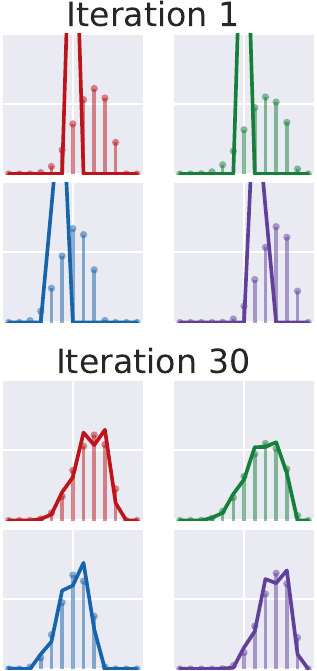}
        \\
    \end{minipage}
    \caption{
        An example run of Sketch-DP as in \cref{fig:walk_through} but using sigmoidal features as shown in
        \textbf{B}. 
        We use $m=13$ sigmoidal features with anchors evenly spaced in $[-4.5, 4.5]$ at locations indicated by the dotted lines. The regression \Eqref{eq:regression} is performed under a densely spaced grid over the white region $[-4, 4]$.
        }
        \vspace{-0.5cm}
    \label{fig:walk_through_sigmoid}
\end{figure*}
To given an example of how to use the 
translation family described in 
\Eqref{eq:translation-family}, we include
another example of Sketch-DP execution in \cref{fig:walk_through_sigmoid}. 
The results are largely similar to using sinusoidal features, except that the evolution is smoother and compared to using sinusoids. We also found that the first two principal components explained $>$97\% variance, 
higher than using sinusoids.
Please see \cref{sec:regression-details,sec:feature-type-details} for how to the heuristics on selecting
the anchors and $\mu$.

\subsection{Extended tabular results}\label{sec:further-tabular-experiments}

We tested Sketch-DP using the following additional MRPs. If unspecified, the default reward distribution for each state is a Dirac delta at 0, and the transition probabilities from a state to its child states are equal. 
\begin{itemize}
    \item Tree: State 1 transitions to states 2 and 3; state 3 transitions to states 4 and 5. 
            State 2 has mean reward 5; state 4 has mean reward -10, and state 5 has mean reward 10. All leaf states are terminal. 
    \item Loopy tree: Same as Tree, but with a connection from state 2 back to state 1.
    \item Cycle: Five states arranged into a cycle, with only a single state having mean reward 1.
    \item Rowland '23: The environment in Example 6.5 of \citet{rowland2023analysis}.
    \item S\&B '18: The environment in Example 6.4 of \citet{sutton2018reinforcement}.
    \item Loopy fork: The environment shown in \cref{fig:walk_through}(A).
\end{itemize}
All environments have discount factor $\gamma=0.9$. For each environment, except Rowland '23 and S\&B'18, the reward distributions for the non-zero-reward states are either Dirac deltas at the specified mean, or Gaussian with specified mean and unit standard deviation. 

The results, extending those in \cref{fig:tabular}, are illustrated
in \cref{fig:tabular_full}.
The results are in general consistent with the main \cref{fig:tabular}. In particular, the 
Cram\'er distances decay as the number of features increases, and can be closer to the corresponding projected ground-truths than the CDRL baseline.

In \cref{sec:feature-type-details}, we suggested that the range of the anchors should be wider than the range of the uniform grid $\mu$ 
on which we measure the regression loss. We validate this intuition by performing another experiment, sweeping the ratio of the width of the anchor range relative to the width of the grid $\mu$, fixing the mid points between these two ranges the same. Here, we use $m=50$ features for each base feature function and apply the default slope described in \cref{sec:feature-type-details}. 
The results in \cref{fig:tabular_anchor_full} shows that a slightly wider anchor range produces reliably small Cram\'er distances for almost all base features.
When the anchor range decreases from 1 to 0, there is a much sharper increase in the Cram\'er distance, because the support on which we impute the 
distribution is too narrow and can miss substantial probability mass outside the support. On the other hand, when the anchor range increases
from 1, the Cram\'er distance also increases because the support points get further away from each other, lowering the resolution of the 
imputed distribution. Since the grid is chosen to be slightly wider than the true return range, we see that the smallest Cram\'er distances
can be attained at anchor range ratio slightly less than 1, but this does not hold universally (see, e.g., random chain and cycle results). 
Choosing this ratio to be slightly greater than 1, as suggested in \cref{sec:feature-type-details,sec:regression-details}, is more reliable
at the cost of a small increase in distributional mismatch.

\begin{figure}
    \centering
    \subfloat[Environments with deterministic rewards, sweeping over feature count $m$.]{
        \includegraphics[width=0.85\textwidth]{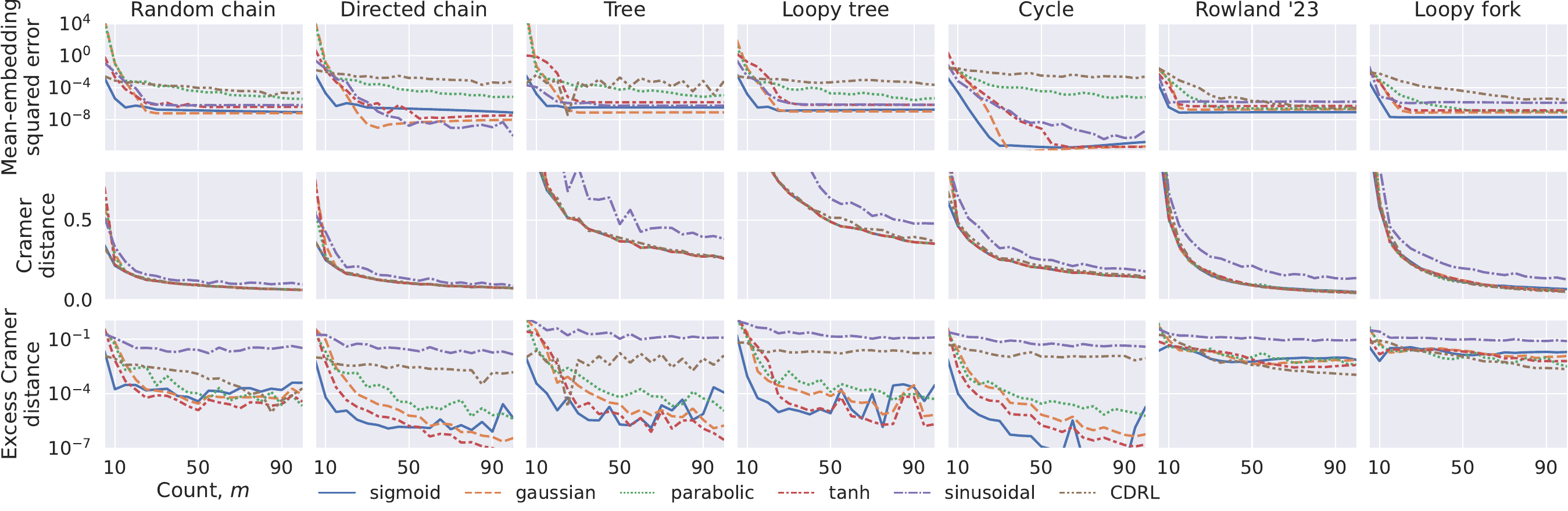}
    }\\
    \subfloat[Environments with deterministic rewards, sweeping over slope $s$.]{
        \includegraphics[width=0.85\textwidth]{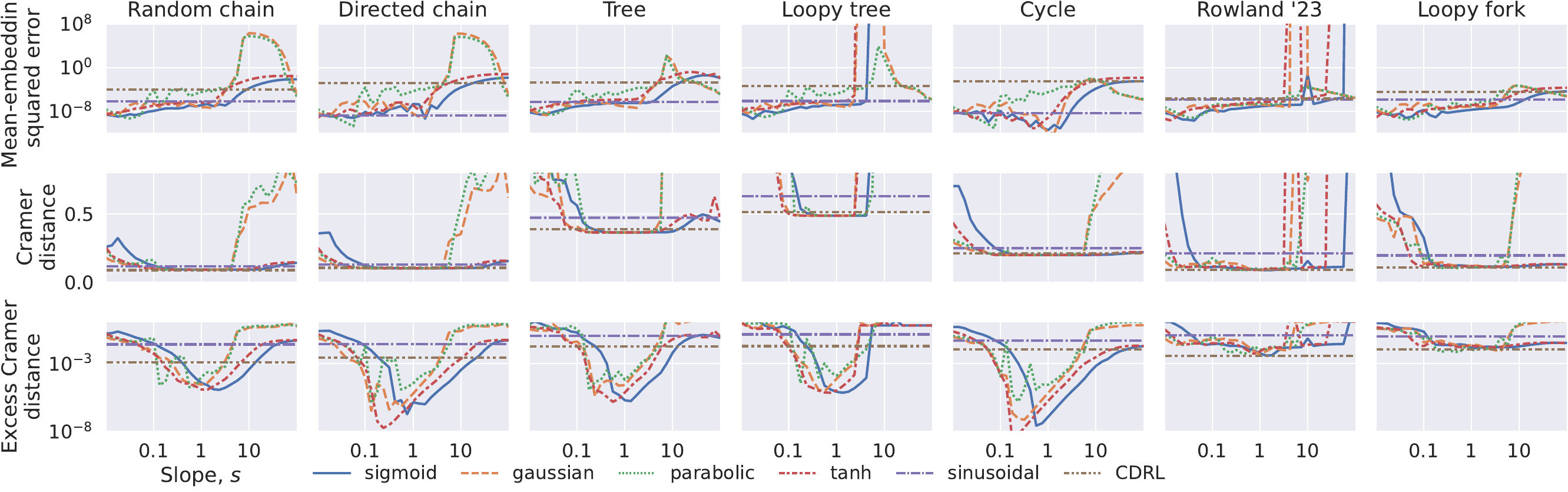}
    }\\
    \subfloat[Environments with Gaussian rewards, sweeping over feature count $m$.]{
        \includegraphics[width=0.85\textwidth]{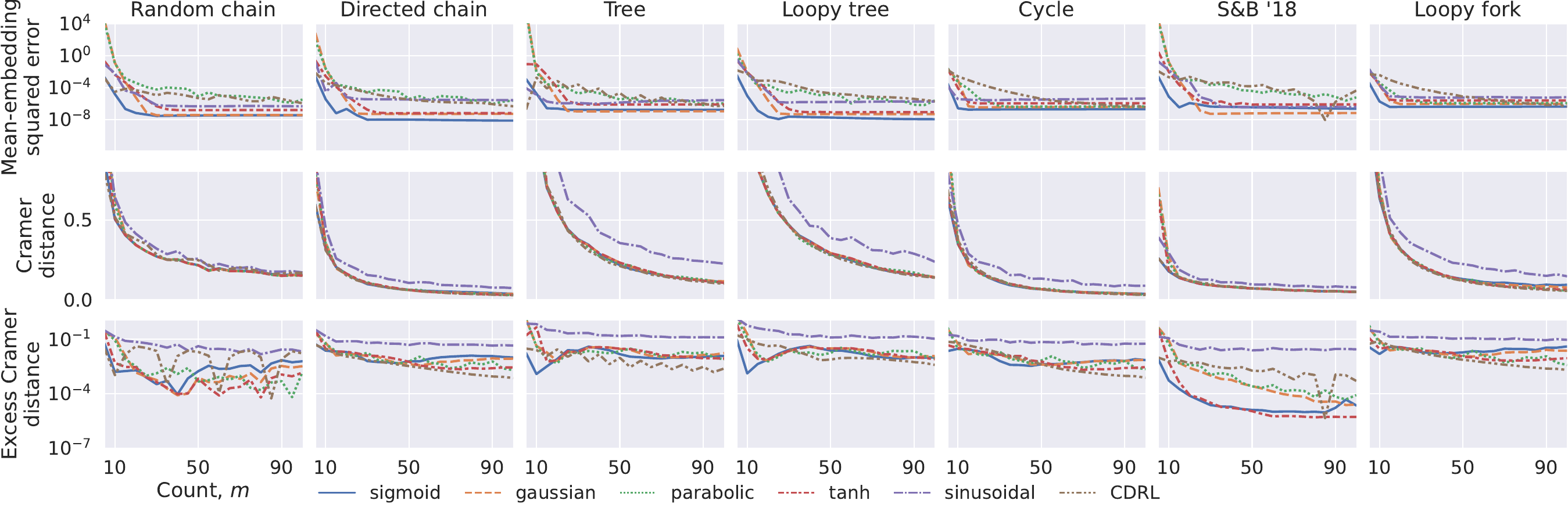}
    }\\
    \subfloat[Environments with Gaussian rewards, sweeping over slope $s$.]{
        \includegraphics[width=0.85\textwidth]{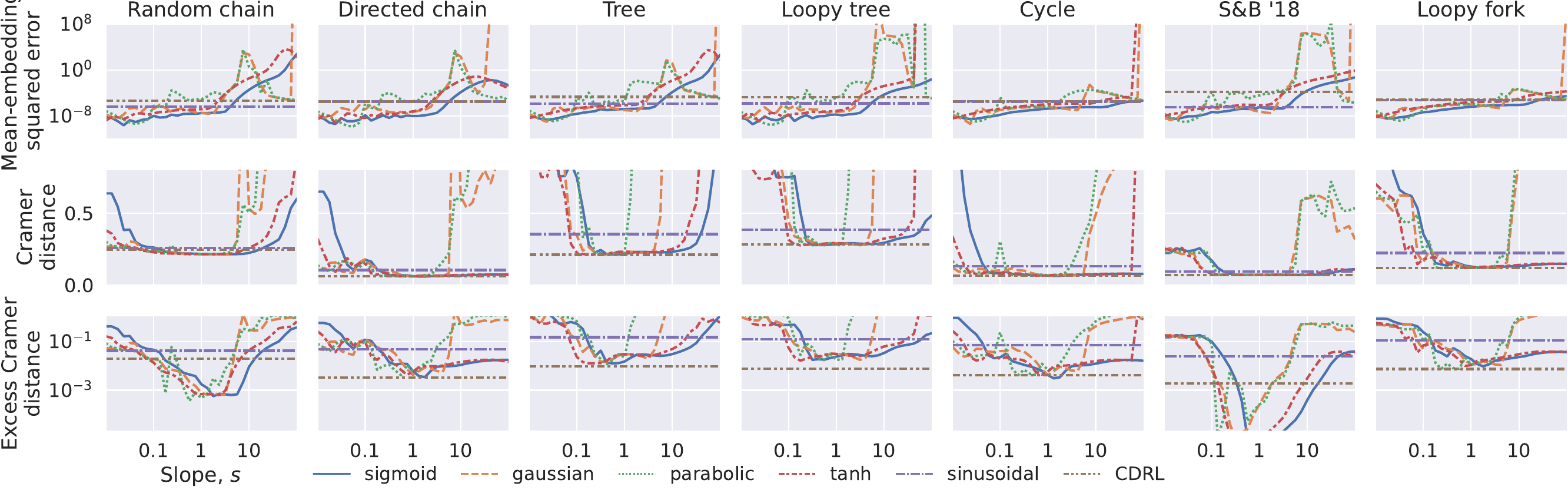}
    }
    \caption{Additional tabular results extending \cref{fig:tabular}.}
    \label{fig:tabular_full}
\end{figure}

\begin{figure}
    \centering
    \subfloat[Environments with deterministic rewards.]{
    \includegraphics[width=0.95\textwidth]{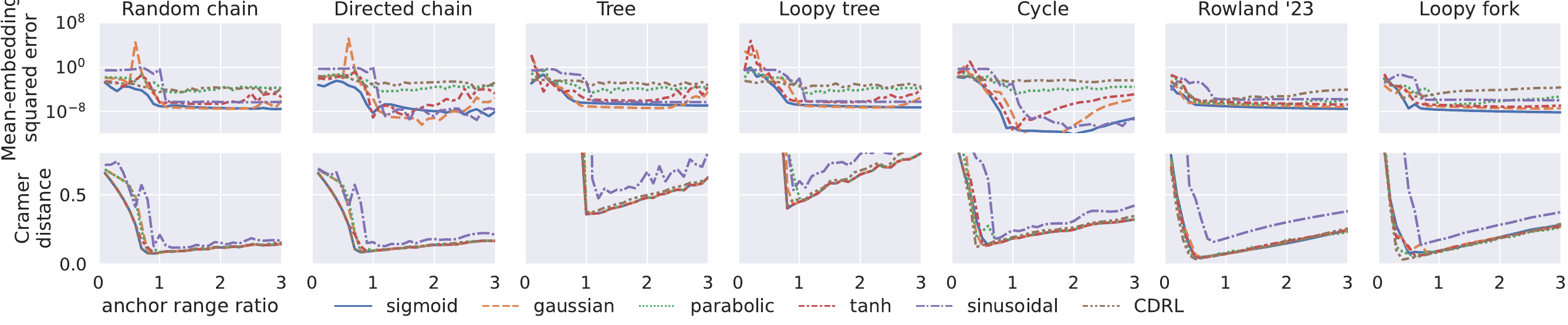}
    }\\
    \subfloat[Environments with Gaussian rewards.]{
    \includegraphics[width=0.95\textwidth]{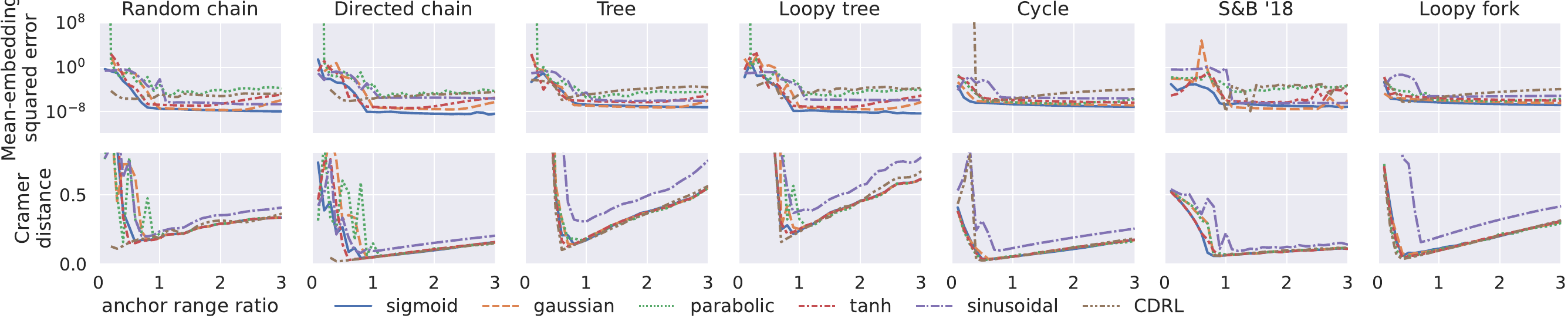}
    }
    \caption{Additional tabular results from sweeping over the range of the anchors relative to the width of the uniform grid given by the support of $\mu$.}
    \label{fig:tabular_anchor_full}
\end{figure}

\subsection{Comparison with statistical functional dynamic programming}\label{sec:sfdp_details}

The Sketch-DP methods developed in this paper were motivated in Section~\ref{sec:background} with the aim of having distributional dynamic programming algorithms that operate directly on sketch values, without the need for the computationally intensive imputation strategies associated with SFDP algorithms. In this section, we empirically compare Sketch-DP and SFDP methods, to quantitatively measure the extent to which this has been achieved. We give details below of the Sketch-DP and SFDP algorithms we compare, and provide comparisons of per-update wallclock time to assess computational efficiency, and distribution reconstruction error to assess accuracy.

\textbf{Sketch-DP.} We consider the Sketch-DP algorithm based on sigmoid features as described in Section~\ref{sec:genreal_sketch_dp_td}, and implemented as described in Section~\ref{sec:tabular} and Appendix~\ref{sec:tabular-details}.

\textbf{SFDP.} We consider the SFDP algorithm for learning expectile values as described by \citet[Section~8.6]{bdr2023}. We use SciPy's default \texttt{minimize} implementation \citep{2020SciPy-NMeth} to solve the imputation strategy optimisation problem given in Equation~(8.15) in \citet{bdr2023}. For a given sketch dimension $m$, we use expectiles at linearly spaced levels $\tau_i = (2i-1)/(2m)$ for $i=1,\ldots,m$.

\textbf{Results.} These two algorithms, for varying numbers of feature/expectiles $m$, were run on 
a selection of deterministic-reward environments as described in \cref{sec:further-tabular-experiments}. 
In \cref{fig:expectile-results}, we plot the Cram\'er distance 
and the excess Cram\'er distance of reconstructed distributions to ground truth, as described in Section~\ref{sec:experiments} and plotted in \cref{fig:tabular}.
In addition, we also plot two wallclock times in each case: the average time it takes to run one iteration of the dynamic programming procedure, and the time it takes to setup the Bellman operator, which includes solving for $B_r$ for Sketch-DP. As predicted, the run time is significantly higher for the SFDP algorithm, due to its use of imputation strategies. The approximation errors measured by Cra\'mer distances are also smaller for Sketch-DP, particularly as the number of features/expectiles is increased. Considering the per-update wallclock times in the third row of the figure, there is consistently a speed up of at least 100x associated with the Sketch-DP algorithm relative to SFDP. This is due to the fact that the Sketch-DP update consists of simple linear-algebraic operations, while the SFDP update includes calls to an imputation strategy, which must solve an optimisation problem. The one-off computation of the Bellman coefficients takes around 0.1--4 seconds, depending on the number of features $m$, which is only at most a couple of SFDP iterations, and hence a small fraction of the total run time of the SFDP algorithm.

\begin{figure}
    \centering
    \includegraphics[width=\textwidth]{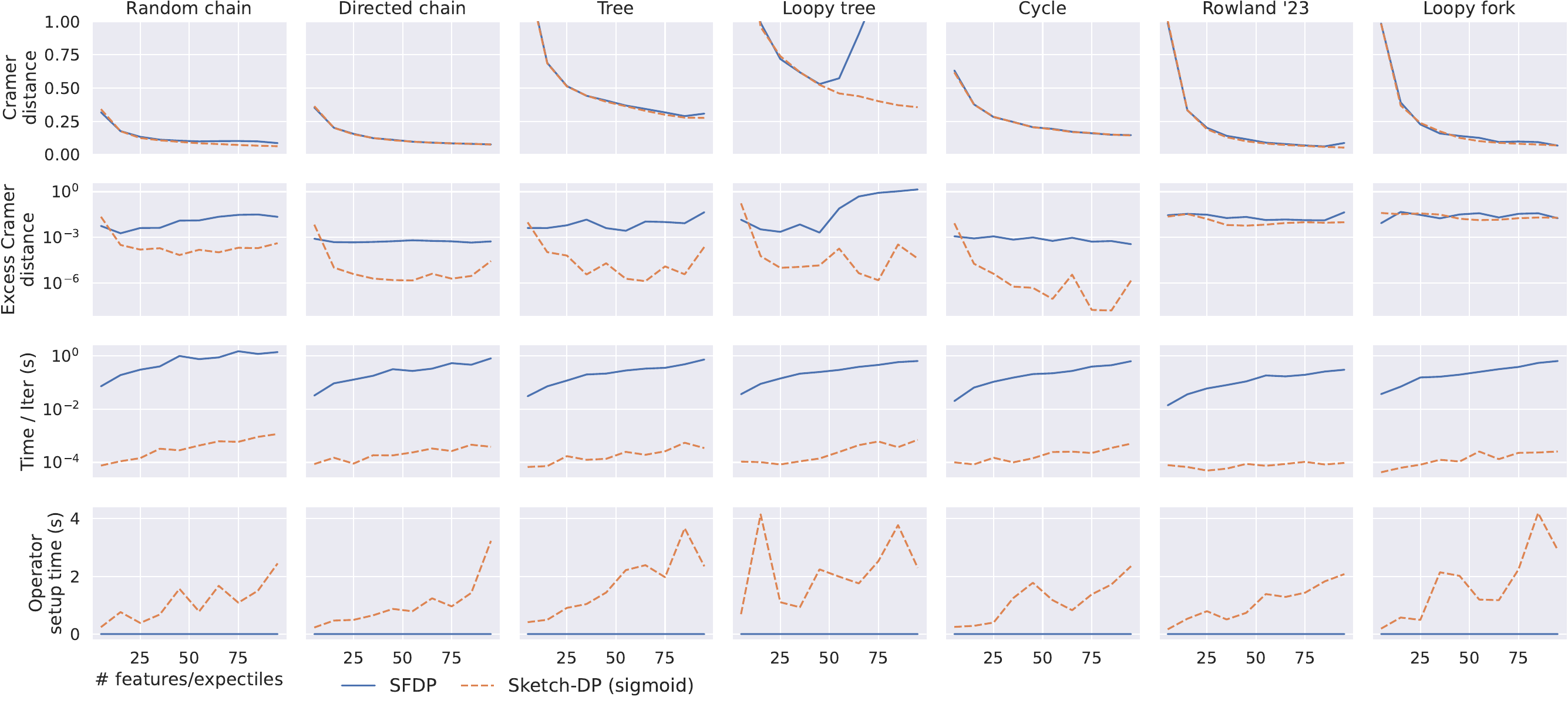}
    \caption{Results comparing Cram\'er distances as in \cref{fig:tabular} (first two rows), and wallclock runtimes for 
    each DP iteration (third row) and for setting the corresponding Bellman operator (bottom row), for Sketch-DP and SFDP algorithms, varying the numbers of features/expectiles $m$.}
    \label{fig:expectile-results}
\end{figure}

\subsection{Extended results on Atari suite}\label{sec:extended-atari}

We show in \cref{fig:atari_diff} the advantage of the Sketch-DQN method to other baselines.
On one hand, we see that Sketch-DQN surpasses DQN on almost all games. Compared to IQN and QR-DQN,
Sketch-DQN is consistently better on 
\textsc{crazy climber}, \textsc{space invaders}, \textsc{river raid}, \textsc{road runners} and \textsc{video pinball}, 
while worse on 
\textsc{assault}, \textsc{asterix}, \textsc{double dunk}, \textsc{krull}, \textsc{phoenix}, and 
\textsc{star gunner}.
 
\begin{figure}
    \centering
    \includegraphics[width=\textwidth]{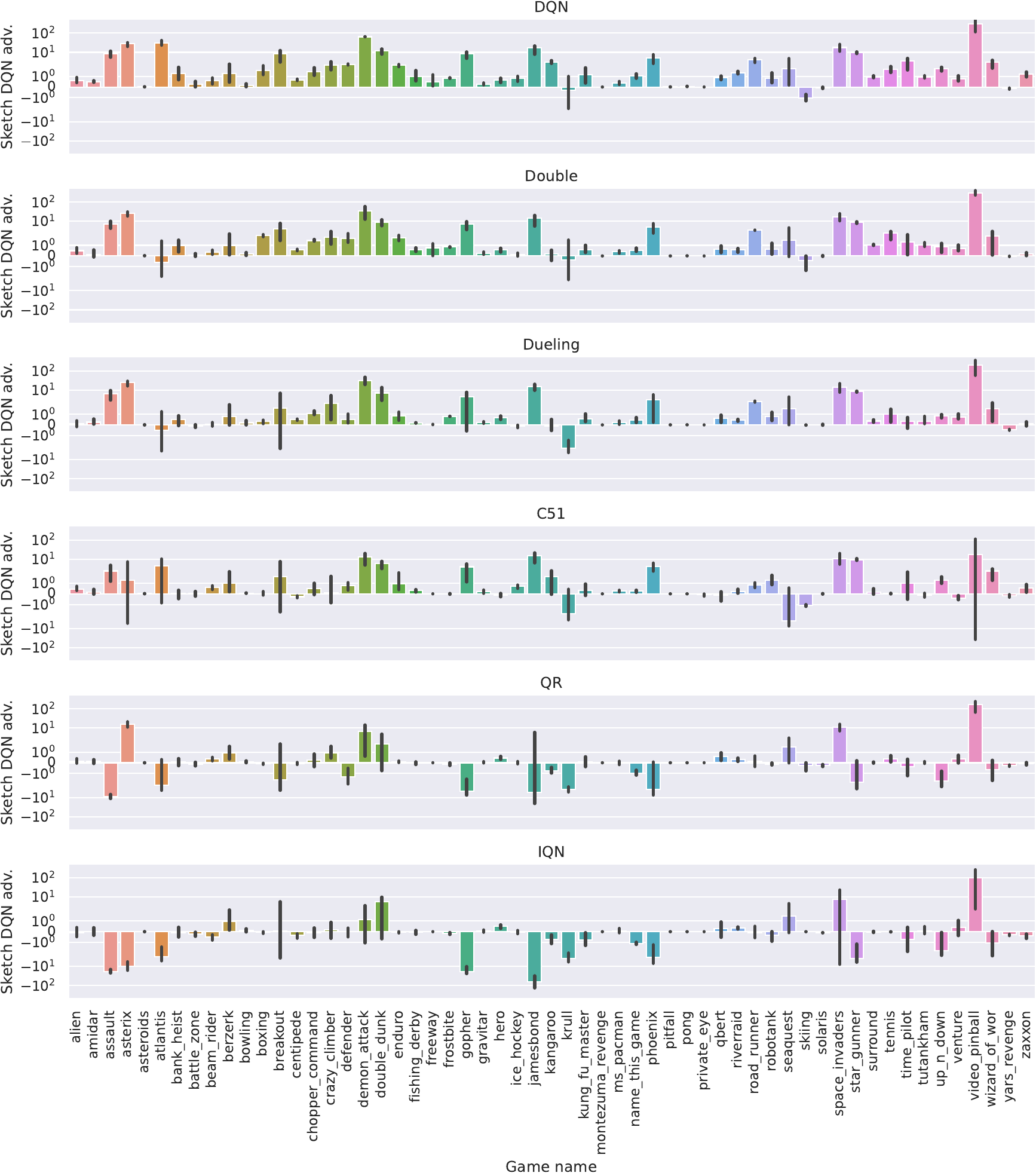}
    \caption{Advantage of Sketch-DQN, measured as Sketch-DQN's normalised return minus
    the returns of other baseline methods. Positive values means Sketch-DQN is better. 
    We show the mean and standard error over 3 seeds for each game.}
    \label{fig:atari_diff}
\end{figure}

To show how sensitive are the results depending on the feature parameters, 
we ran the full Atari suite using a few sigmoidal and Gaussian features, and show the results in \cref{fig:atari_sweep}. Overall, we see that the choice on feature parameters is important; 
in particular, the sigmoidal feature outperforms Gaussian features. For the sigmoidal feature, the performance 
improved from using 101 to 201 features. On the contrary, for Gaussian features, 
increasing the feature count does not produce much change. 
\begin{figure}
    \centering
    \subfloat[Mean normalised return.]{\includegraphics[width=\textwidth]{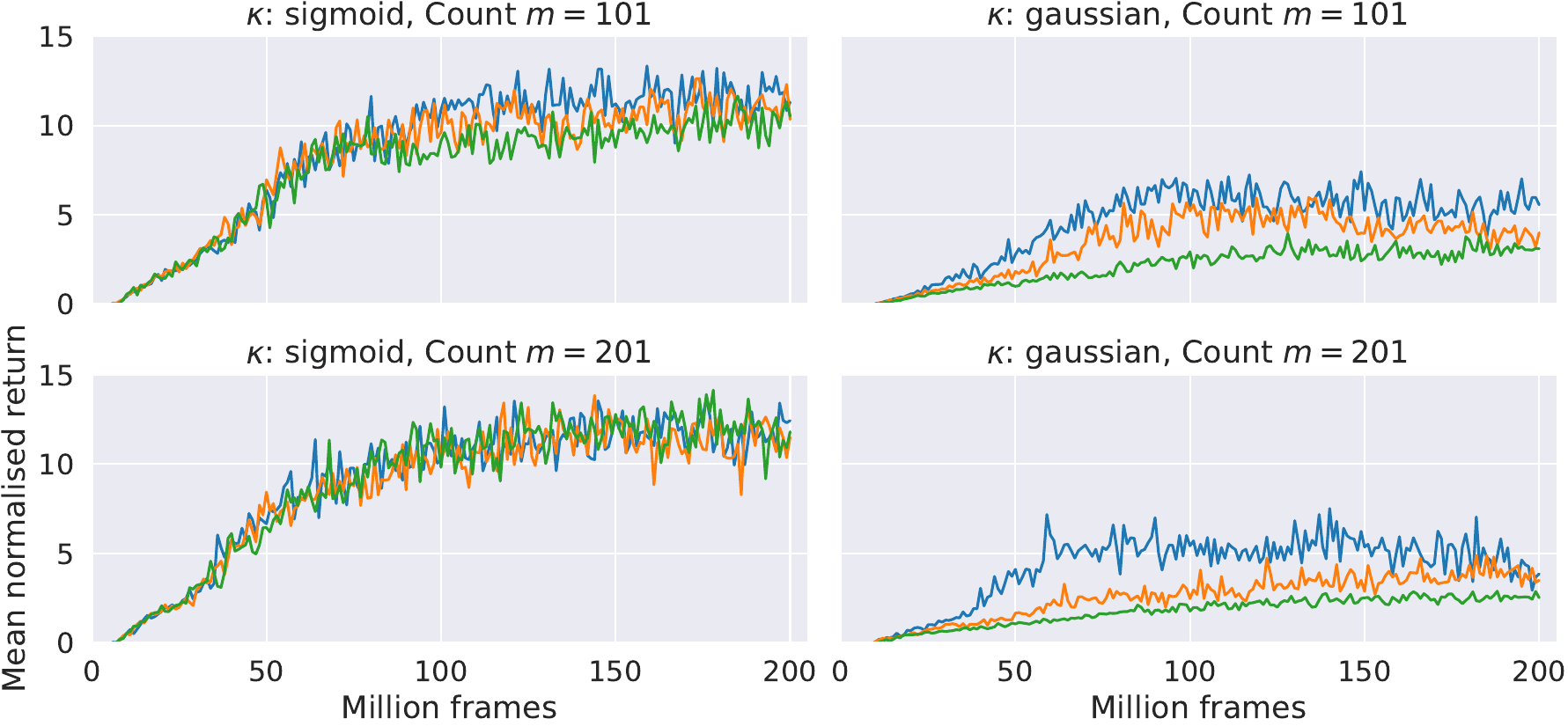}}\\
    \subfloat[Median normalised return.]{\includegraphics[width=\textwidth]{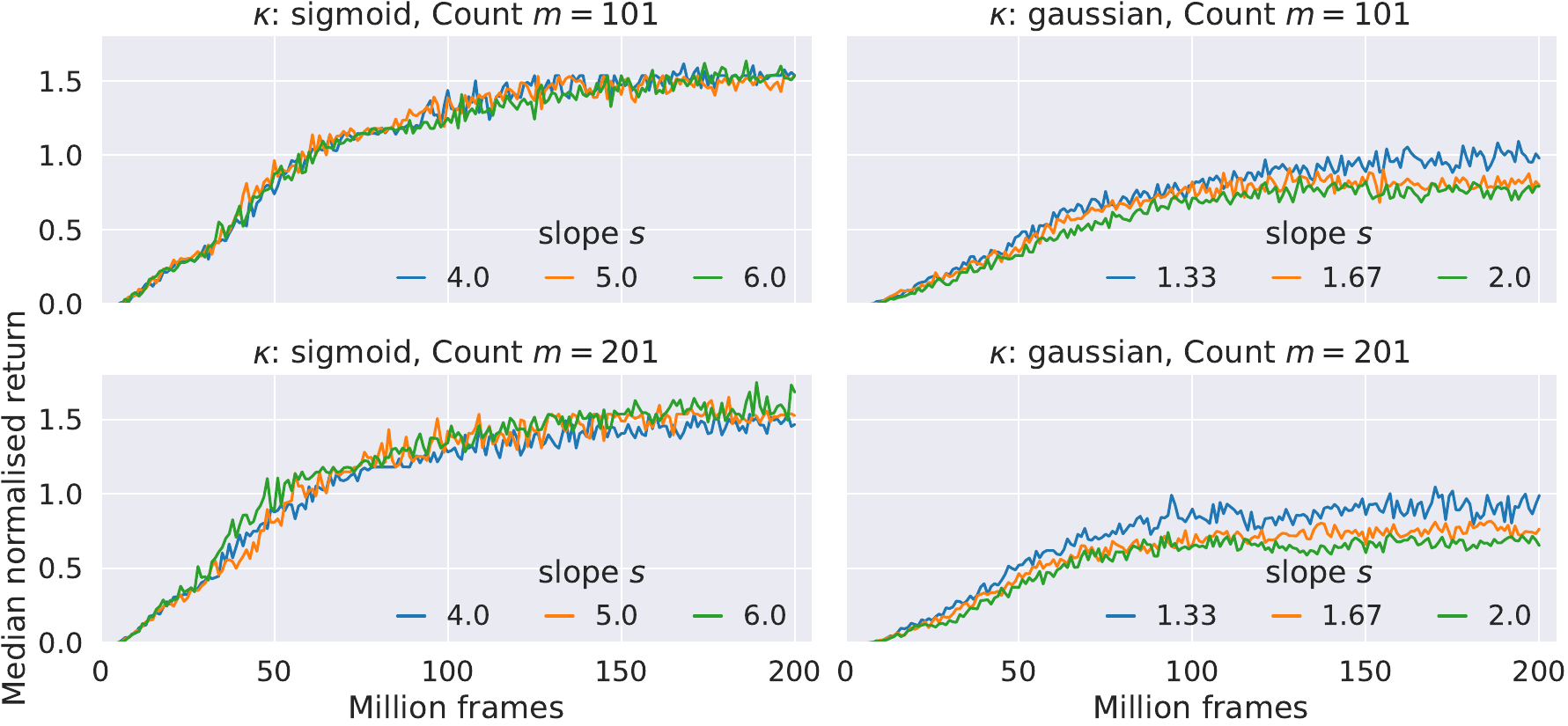}}
    \caption{Results on Atari suite for different feature parameters.}
    \label{fig:atari_sweep}
\end{figure}

\subsection{Runtime comparison}\label{sec:runtime_details}

Here, we report the mean (±s.d.) rate (per second) at which frames are processed during training in our Atari experiments, with each agent running on a single V100 GPU. These frame-processing times reflect the wallclock time associated with all aspects of the DQN training, including network forward passes for action selection, environment simulation, and periodic gradient updates. These statistics are averaged across all games and seeds.

The results are shown in \cref{tab:dqn_framerate}. Note that, by default, C51 uses 51 atoms, QR-DQN uses 201 quantiles, and IQN uses 64 quantiles.
\begin{table}[t]
\centering
\begin{tabular}{c|cccc}
           & \multicolumn{4}{c}{Number of features / atoms} \\
Method     &    51              &64             & 201               & 401          \\ \hline
Sketch-DQN &  -                 &-              & 1326 ± 107        & 1320 ± 110   \\
Categorical-DQN& (C51)~1309 ± 146   & 1306 ± 139    & 1293 ± 152        & 1291 ± 154\\
QR-DQN     &  -                 &-              & 1258 ± 107        & 1256 ± 106   \\
IQN        &  -                 & 1120 ± 90     & 698 ± 41          & 400 ± 16     \\
\end{tabular}
\caption{Frame rate for selected methods. Higher is faster.}
\label{tab:dqn_framerate}
\end{table}
The Sketch-DQN method has the highest frame rate, and C51 and QR-DQN has slightly lower average frame rates. IQN with default 64 quantiles has a slightly lower average frame rate still, and is much lower when the number of quantiles is increased to match the number of predictions made by QR-DQN and Sketch-DQN. This is because the IQN architecture requires one forward pass through the MLP component of the network for each predicted quantile level. By contrast, the Sketch \& QR architectures simply modify the original DQN architecture to produce multiple predictions of sketch values/quantiles from the final hidden layer of the network.

\end{document}